\pdfoutput=1
\documentclass{article}

\PassOptionsToPackage{square,numbers}{natbib}
\usepackage[final]{neurips_2024}

\usepackage[utf8]{inputenc} %

\usepackage{graphicx}

\usepackage[T1]{fontenc}

\IfFileExists{headers/config/showoverfull.config}{
	\overfullrule=1cm
}{
}

\usepackage{marginnote}

\usepackage[backgroundcolor=none,linecolor=red,textsize=footnotesize]{todonotes}

\usepackage{etoolbox}

\newbool{includeappendix}
\setbool{includeappendix}{true} %
\IfFileExists{headers/config/noappendix.config}{
	\setbool{includeappendix}{false}
}{}

\newif\ifincludeappendixx
\ifbool{includeappendix}{
	\includeappendixxtrue
}{
	\includeappendixxfalse
}

\usepackage{xr} %
\usepackage{filecontents}

\ifbool{includeappendix}{}{
	\input{appendix-labels-loader}

	\externaldocument{appendix-labels}
}

\usepackage{acro} %

\usepackage{listings}

\usepackage{textcomp}

\usepackage{xcolor}

\usepackage[scaled=0.8]{beramono}

\definecolor{ckeyword}{HTML}{7F0055}
\definecolor{ccomment}{HTML}{3F7F5F}
\definecolor{cstring}{HTML}{2A0099}

\lstdefinestyle{numbers}{
	numbers=left,
	framexleftmargin=20pt,
	numberstyle=\tiny,
	firstnumber=auto,
	numbersep=1em,
	xleftmargin=2em
}

\lstdefinestyle{layout}{
	frame=none,
	captionpos=b,
}

\lstdefinestyle{comment-style}{
	morecomment=[l]//,
	morecomment=[s]{/*}{*/},
	commentstyle={\color{ccomment}\itshape},
}

\lstdefinestyle{string-style}{
	morestring=[b]",%
	morestring=[b]',%
	stringstyle={\color{cstring}},
	showstringspaces=false,%
}

\lstdefinestyle{keyword-style}{
	keywordstyle={\ttfamily\bfseries},
	morekeywords={
		function,
		constructor,
		int,
		bool,
		return,
		returns,
		uint
	},
	morekeywords = [2]{},
	keywordstyle = [2]{\text},
	sensitive=true,
}

\lstdefinestyle{input-encoding}{
	inputencoding=utf8,
	extendedchars=true,
	literate=
	{ℝ}{$\reals$}1%
	{→}{$\rightarrow$}1%
	{α}{$\alpha$}1%
	{β}{$\beta$}1%
	{λ}{$\lambda$}1%
	{θ}{$\theta$}1%
	{ϕ}{$\phi$}1%
}

\lstdefinestyle{escaping}{
	moredelim={**[is][\color{blue}]{\%}{\%}},
	escapechar=|,
	mathescape=true
}

\lstdefinestyle{default-style}{
	basicstyle=\fontencoding{T1}\ttfamily\footnotesize,
	style=numbers,
	style=layout,
	style=comment-style,
	style=string-style,
	style=keyword-style,
	style=input-encoding,
	style=escaping,
	tabsize=2,
	upquote=true
}

\lstdefinelanguage{BASIC}{
	language=C++,
	style=default-style
}[keywords,comments,strings]%

\usepackage{microtype}
\usepackage{graphicx}
\usepackage{subcaption}
\usepackage{booktabs} %
\usepackage{hyperref}

\usepackage{amsmath}
\usepackage{amssymb}
\usepackage{amsfonts}
\usepackage{mathtools}
\usepackage{amsthm}
\usepackage{thm-restate}
\usepackage{dsfont}
\usepackage{xspace}
\usepackage{xcolor}
\usepackage{siunitx}
\usepackage{dcolumn}
\newcolumntype{d}[1]{D{.}{.}{#1}}

\usepackage[T1]{fontenc}    %
\usepackage{url}            %
\usepackage{nicefrac}       %
\usepackage{wrapfig}

\usepackage{algorithm}
\usepackage{algpseudocode}
\algtext*{EndFunction}

\theoremstyle{plain}
\newtheorem{theorem}{Theorem}[section]

\newtheorem{lemma}[theorem]{Lemma}

\theoremstyle{definition}
\newtheorem{definition}[theorem]{Definition}

\theoremstyle{remark}

\usepackage{tensor}
\usepackage{bm}

\usepackage{amsmath,amsfonts,bm}

\def\eqref#1{equation~\ref{#1}}

\def\1{\bm{1}}

\def\vb{{\bm{b}}}

\def\ve{{\bm{e}}}
\def\vf{{\bm{f}}}

\def\vq{{\bm{q}}}

\def\vx{{\bm{x}}}
\def\vy{{\bm{y}}}
\def\vz{{\bm{z}}}

\def\mA{{\bm{A}}}
\def\mB{{\bm{B}}}

\def\mI{{\bm{I}}}

\def\mL{{\bm{L}}}

\def\mQ{{\bm{Q}}}
\def\mR{{\bm{R}}}
\def\mS{{\bm{S}}}

\def\mU{{\bm{U}}}
\def\mV{{\bm{V}}}
\def\mW{{\bm{W}}}
\def\mX{{\bm{X}}}
\def\mY{{\bm{Y}}}
\def\mZ{{\bm{Z}}}

\DeclareMathAlphabet{\mathsfit}{\encodingdefault}{\sfdefault}{m}{sl}
\SetMathAlphabet{\mathsfit}{bold}{\encodingdefault}{\sfdefault}{bx}{n}

\newcommand{\E}{\mathbb{E}}

\newcommand{\R}{\mathbb{R}}

\DeclareMathOperator{\diag}{diag}

\DeclareMathOperator{\sparsity}{sparsity}

\DeclareMathOperator{\rank}{rank}
\DeclareMathOperator{\Span}{span}
\DeclareMathOperator{\relu}{ReLU}

\newcommand{\bc}[1]{\mathcal{#1}}
\renewcommand{\P}{\mathds{P}}

\newcommand{\hbc}{honest-but-curious\xspace}
\newcommand{\tool}{\textsc{SPEAR}\xspace}
\newcommand{\toolL}{\textbf{Sp}arsity \textbf{E}xploiting \textbf{A}ctivation \textbf{R}ecovery}

\newcommand{\computesigma}{\textsc{ComputeLambda}\xspace}
\newcommand{\greedyfilter}{\textsc{GreedyFilt}\xspace}
\newcommand{\nameQ}{disaggregation matrix\xspace}

\newcommand{\cifar}{CIFAR-10\xspace}
\newcommand{\TIN}{\textsc{TinyImageNet}\xspace}
\newcommand{\mnist}{\textsc{MNIST}\xspace}
\newcommand{\ImgNet}{\textsc{ImageNet}\xspace}

\renewcommand{\th}{$^\text{th}$\xspace}

\usepackage{tikz}
\usetikzlibrary{calc,decorations,decorations.pathmorphing}
\usetikzlibrary{positioning,fit,arrows}
\usetikzlibrary{decorations.markings}
\usetikzlibrary{shapes,shapes.geometric}
\usetikzlibrary{shadows,patterns,snakes}
\usetikzlibrary{backgrounds,decorations.pathreplacing,automata}
\usetikzlibrary{intersections}
\usetikzlibrary{angles,quotes}
\usetikzlibrary{plotmarks}
\usetikzlibrary{patterns}
\usetikzlibrary{arrows.meta}
\usetikzlibrary{shapes.misc}
\usetikzlibrary{chains}

\usepackage[capitalize,noabbrev]{cleveref}

\hypersetup{
	colorlinks,
	linkcolor={black},
}

\crefformat{section}{\S#2#1#3}

\crefrangeformat{section}{\S#3#1#4\crefrangeconjunction\S#5#2#6}

\crefmultiformat{section}{\S#2#1#3}{\crefpairconjunction\S#2#1#3}{\crefmiddleconjunction\S#2#1#3}{\creflastconjunction\S#2#1#3}

\newcommand{\crefrangeconjunction}{--}

\crefname{listing}{Lst.}{listings}
\crefname{line}{Line}{Lines}
\crefname{appendix}{App.}{App.}

\newcommand{\appref}[1]{%
	\ifbool{includeappendix}{\cref{#1}}{the appendix}%
}
\newcommand{\Appref}[1]{%
	\ifbool{includeappendix}{\cref{#1}}{The appendix}%
}

\definecolor{refcolor}{RGB}{23, 120, 108} 
\definecolor{citcolor}{RGB}{23, 120, 18}

\crefformat{figure}{Fig.~#2{\color{refcolor}#1}#3}
\Crefformat{figure}{Fig.~#2{\color{refcolor}#1}#3}
\crefformat{algorithm}{Algorithm~#2{\color{refcolor}#1}#3}
\crefformat{line}{Line~#2{\color{refcolor}#1}#3}
\Crefformat{line}{Line~#2{\color{refcolor}#1}#3}
\crefrangeformat{line}{Lines~#3{\color{refcolor}#1}#4--#5{\color{refcolor}#2}#6}
\crefrangeformat{figure}{Fig.~#3{\color{refcolor}#1}#4--#5{\color{refcolor}#2}#6}
\crefformat{appendix}{App.~#2{\color{refcolor}#1}#3}
\Crefformat{appendix}{App.~#2{\color{refcolor}#1}#3}
\crefformat{algorithm}{Alg.~#2{\color{refcolor}#1}#3}
\Crefformat{algorithm}{Alg.~#2{\color{refcolor}#1}#3}
\crefformat{equation}{Eq.~#2{\color{refcolor}#1}#3} %
\Crefformat{equation}{Eq.~#2{\color{refcolor}#1}#3} %
\crefformat{section}{Sec.~#2{\color{refcolor}#1}#3}
\Crefformat{section}{Sec.~#2{\color{refcolor}#1}#3}
\crefformat{table}{Table~#2{\color{refcolor}#1}#3}
\Crefformat{table}{Table~#2{\color{refcolor}#1}#3}
\crefrangeformat{table}{Tables~#3{\color{refcolor}#1}#4 to~#5{\color{refcolor}#2}#6}
\Crefrangeformat{table}{Tables~#3{\color{refcolor}#1}#4 to~#5{\color{refcolor}#2}#6}

\hypersetup{citecolor=citcolor}

\title{\tool: Exact Gradient Inversion of Batches in Federated Learning}

\author{Dimitar I. Dimitrov$^{1}$,$\quad$ Maximilian Baader$^{2}$,$\quad$ Mark Niklas Müller$^{2,3}$,$\quad$ Martin Vechev$^{2}$\hfil\\
	\hspace{0em}
	$^{1}$ INSAIT, Sofia University "St. Kliment Ohridski"
	\hspace{2em}
	$^{2}$ ETH Zurich
	\hspace{2em}$^{3}$ LogicStar.ai
	\hfil\\
	\texttt{\{dimitar.iliev.dimitrov\}@insait.ai} $^{1}$\hfil\\
	\texttt{\{mbaader, mark.mueller, martin.vechev\}@inf.ethz.ch}  $^{2}$\hfil\\
}

\begin{document}

	\maketitle
	
	\vspace{-1.5em}
\begin{abstract}
Federated learning is a framework for collaborative machine learning where clients only share gradient updates and not their private data with a server. However, it was recently shown that gradient inversion attacks can reconstruct this data from the shared gradients. In the important honest-but-curious setting, existing attacks enable exact reconstruction only for batch size of $b=1$, with larger batches permitting only approximate reconstruction. In this work, we propose \tool, \emph{the first algorithm reconstructing whole batches with $b >1$ exactly}. \tool combines insights into the explicit low-rank structure of gradients with a sampling-based algorithm. Crucially, we leverage ReLU-induced gradient sparsity to precisely filter out large numbers of incorrect samples, making a final reconstruction step tractable. We provide an efficient GPU implementation for fully connected networks and show that it recovers high-dimensional ImageNet inputs in batches of up to $b \lesssim 25$ exactly while scaling to large networks. Finally, we show theoretically that much larger batches can be reconstructed with high probability given exponential time.
\end{abstract}
\vspace{-0.5em}
	\section{Introduction}

\begin{wrapfigure}{R}{0.5\textwidth}
	\begin{minipage}{.5\textwidth}
		\vspace{-4.3em}
		\begin{figure}[H]
			\centering
			\hspace{-0.5cm}
			\input{figures/comparison_fig.tex}
			\caption{A sample of four images from a batch of $b=20$, reconstructed using our \tool (top) or the prior state-of-the-art \citet{geiping} (mid), compared to the ground truth (bottom).}
			\label{fig:accept}
		\end{figure}
		\vspace{-3em}
	\end{minipage}%
\end{wrapfigure}

\paragraph{Federated Learning} has emerged as the dominant paradigm for training machine learning models collaboratively without sharing sensitive data \citep{fedsgd}. 
Instead, a central server sends the current model to all clients which then send back gradients computed on their private data. The server aggregates the gradients and uses them to update the model. Using this approach sensitive data never leaves the clients' machines, aligning it better with data privacy regulations such as the General Data Protection Regulation (GDPR) and California Consumer Privacy Act (CCPA).

\paragraph{Gradient Inversion Attacks} Recent work has shown that an \hbc server can use the shared gradient updates to recover the sensitive client data \citep{dlg,fedavg}. However, while \emph{exact} reconstruction was shown to be possible for batch sizes of $b=1$ \citep{analyticPhong,rgap}, it was assumed to be infeasible for larger batches.  This led to a line of research on approximate methods that sacrificed reconstruction quality in order to recover batches of $b>1$ inputs \cite{aaai,tableak,cocktail}. In this paper we challenge this fundamental assumption and, for the first time, show that exact reconstruction is possible for batch sizes $b > 1$.

\paragraph{This Work: Exact Reconstruction of Batches} We propose the \emph{first gradient inversion attack reconstructing inputs exactly for batch sizes $b > 1$} in the \hbc setting. In \cref{fig:accept}, we show the resulting reconstructions versus approximate methods~\citep{geiping} for a batch of $b=20$ images.

Our approach leverages two key properties of gradient updates in fully connected ReLU networks: 
First, these gradients have a specific \emph{low-rank structure} due to small batch sizes $b \ll n,m$ compared to the input dimensionality $n$ and the hidden dimension $m$. 
Second, the (unknown) gradients with respect to the inputs of the first ReLU layer are sparse due to the ReLU function itself. 
We combine these properties with ideas from sparsely-used dictionary learning \citep{spielman2012exact} to propose a sampling-based algorithm, called \tool (\toolL) and show that it succeeds with high probability for $b < m$. While \tool scales exponentially with batch size $b$, we provide a highly parallelized GPU implementation, which empirically allows us to reconstruct batches of size up to $b \lesssim 25$ exactly even for large inputs (\ImgNet) and networks (widths up to $2000$ neurons and depths up to $9$ layers) in around one minute per batch.

\paragraph{Main Contributions:}
\begin{itemize}
    \vspace{-3.5mm}
    \setlength\itemsep{0.15em}
    \item The first gradient inversion attack showing theoretically that \emph{exact reconstruction} of complete batches with size $b \!>\! 1$ in the \hbc setting is possible.
    \item \tool: a sampling-based algorithm leveraging \emph{low rankness} and ReLU-induced \emph{sparsity of gradients} for exact gradient inversion that succeeds with high probability.
    \item A highly parallelized GPU implementation of \tool, which we empirically demonstrate to be effective across a wide range of settings and make publicly available on \href{https://github.com/eth-sri/SPEAR}{GitHub}.
\end{itemize}

	\section{Method Overview}\label{sec:overview}

\begin{wrapfigure}{R}{0.565\textwidth}
	\begin{minipage}{.565\textwidth}
		\vspace{-5.7em}
		\begin{figure}[H]
			\centering
			\includegraphics[width=\linewidth]{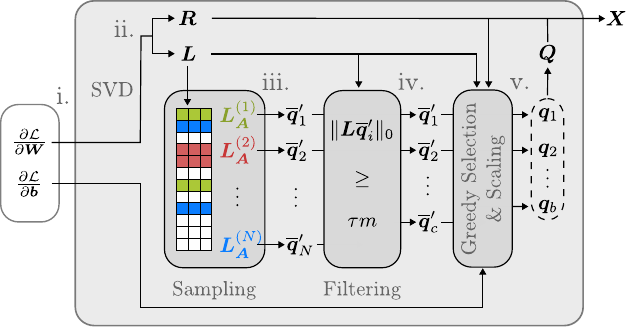}
			\captionof{figure}{Overview of \tool. The gradient $\frac{\partial \mathcal{L}}{\mW}$ is decomposed to $\mR$ and $\mL$. Sampling gives $N$ proposal directions, which we filter down to $c$ candidates via a sparsity criterion with threshold $\tau*m$. A greedy selection method selects batchsize $b$ directions. Scale recovery via $\frac{\partial \mathcal{L}}{\partial \vb}$ returns the \nameQ $\mQ$ and thus the inputs $\mX$.}
			\label{fig:overview}
		\end{figure}
		\vspace{-3.3em}
	\end{minipage}
\end{wrapfigure}

We first introduce our setting before giving a high-level overview of our attack \tool, whose sketch is shown in \cref{fig:overview}. 

\paragraph{Setting} 
We consider a neural network $\vf$ containing a linear layer $\vz = \mW \vx + \vb$ followed by ReLU activations $\vy = \relu(\vz)$ trained with a loss function $\mathcal{L}$. Let now $\mX \in \mathbb{R}^{n \times b}$ be a batch of $b$ inputs to the linear layer $\mZ=\mW \mX + (\vb| \dots| \vb)$, with weights $\mW \in \mathbb{R}^{m \times n}$, bias $\vb \in \mathbb{R}^m$ and output $\mZ \in \mathbb{R}^{m \times b}$. Further, let $\mY \in \mathbb{R}^{m \times b}$ be the result of applying the ReLU activation to $\mZ$, i.e., $\mY = \relu(\mZ)$ and assume $b \leq m, n$. The goal of \tool is to \emph{recover the inputs} $\mX$ (up to permutation) given the gradients $\frac{\partial \mathcal{L}}{\partial \mW}$ and $\frac{\partial \mathcal{L}}{\partial \vb}$ (see \cref{fig:overview},~\romannumeral 1).

\paragraph{Low-Rank Decomposition}
We first show that the weight gradient $\frac{\partial \mathcal{L}}{\partial \mW} = \frac{\partial \mathcal{L}}{\partial \mZ} \mX^\top$ naturally has a low rank $b \leq m, n$ (\cref{theorem:dLdW_dLdZ.XT}) and can therefore be decomposed as \makebox{$\frac{\partial \mathcal{L}}{\partial \mW} = \mL \mR$} with $\mL \in \R^{m \times b}$ and $\mR \in \R^{b \times n}$ using SVD (\cref{fig:overview},~\romannumeral 2). We then prove the existence \nameQ $\mQ = (\vq_1 | \dots | \vq_b) \in \text{GL}_b(\mathbb{R})$, allowing us to express the inputs as $\mX^\top = \mQ^{-1} \mR$ and activation gradients as $\frac{\partial \mathcal{L}}{\partial \mZ} = \mL \mQ$ (\cref{theorem:q}). 
Next, we leverages the sparsity of $\frac{\partial \mathcal{L}}{\partial \mZ}$ to recover $\mQ$ exactly.

\paragraph{ReLU Induced Sparsity}
We show that ReLU layers induce sparse activation gradients $\frac{\partial \mathcal{L}}{\partial \mZ}$ (\cref{sec:sparsity}). We then leverage this sparsity to show that, with high probability, there exist submatrices $\mL_\mA \in \mathbb{R}^{b-1 \times b}$ of $\mL$, such that their kernel is an unscaled column $\overline{\vq}_i$ of our \nameQ $\mQ$, i.e., $\ker(\mL_\mA) = \Span(\vq_i)$, for all $i \in \{1, \dots, b\}$ (\cref{theorem:dirs}). 
Given these unscaled colmuns $\overline{\vq}_i$, we recover their scale by leveraging the bias gradient $\frac{\partial \mathcal{L}}{\partial \vb}$ (\cref{theorem:scaling}). %

\paragraph{Sampling and Filtering Directions}
To identify the submatrices $\mL_\mA$ of $\mL$ which induce the directions $\overline{\vq}_i$, we propose a sampling approach (\cref{sec:filtering}): We randomly sample $b-1$ rows of $\mL$ to obtain an $\mL_\mA$ and thus proposal direction $\overline{\vq}'_i = \ker(\mL_\mA)$ (\cref{fig:overview}~\romannumeral 3). 
Crucially, the product $\mL \overline{\vq}'_i = \frac{\partial \mathcal{L}}{\partial \vz_i}$ recovers a column of the sparse activation gradient $\frac{\partial \mathcal{L}}{\partial \mZ}$ for correct directions $\overline{\vq}'_i$ and a dense linear combination of such columns for incorrect ones. This sparsity gap allows the large number $N$ of proposal directions obtained from submatrices $\mL_\mA$ to be filtered to $c \gtrsim b$ unique candidates (\cref{fig:overview}~\romannumeral 4).
\paragraph{Greedy Direction Selection}
We now have to select the correct $b$ directions from our set of $c$ candidates (\cref{fig:overview},~\romannumeral 5). To this end, we build an initial solution $\mQ'$ from the $b$ directions inducing the highest sparsity in $\frac{\partial \mathcal{L}}{\partial \mZ}' = \mL \mQ'$. To assess the quality of this solution $\mQ'$, we introduce the sparsity matching score $\sigma$ which measures how well the sparsity of the activation gradients $\frac{\partial \mathcal{L}}{\partial \mZ}'$ matches the ReLU activation pattern induced by the reconstructed input $\mX'^\top = \mQ'^{-1} \mR$. Finally, we greedily optimize $\mQ'$ to maximize the sparsity matching score, by iteratively replacing an element $\vq'_i$ of $\mQ'$ with the candidate direction $\vq'_j$ yielding the greatest improvement in $\sigma$ until convergence. We can then validate the resulting input $\mX^\top = \mQ^{-1} \mR$ by checking whether it induces the correct gradients. We formalize this as \cref{alg:gradient_dismantle} in \cref{sec:algorithm} and show that it succeeds with high probability for $b < m$.

	\section{Gradient Inversion via Sparsity and Low-Rankness}
\label{sec:technical}

In this section, we will demonstrate that both low rankness and sparsity arise naturally for gradients of fully connected ReLU networks and explain theoretically how we recover $\mX$. 
Specifically, in \cref{sec:low_rank}, we first argue that $\frac{\partial \mathcal{L}}{\partial \mW} = \frac{\partial \mathcal{L}} {\partial \mZ}  \mX^T$ follows direclty from the chain rule. We then show that for every decomposition $\frac{\partial \mathcal{L}}{\partial \mW} = \mL \mR$, there exists an unknown \nameQ $\mQ$ allowing us to reconstruct $\mX^\top = \mQ^{-1} \mR$ and  $\frac{\partial \mathcal{L}} {\partial \mZ} = \mL \mQ$. 
The remainder of the section then focuses on recovering $\mQ$.
To this end, we show in \cref{sec:sparsity} that ReLU layers induce sparsity in $\frac{\partial \mathcal{L}}{\partial \mZ}$, which we then leveraged in \cref{sec:breaking_aggregation} to reconstruct the columns of $\mQ$ up to scale. Finally, in \cref{sec:scale}, we show how the scale of $\mQ$'s columns can be recovered from $\frac{\partial \mathcal{L}}{\partial \vb}$. 
Unless otherwise noted, we defer all proofs to \cref{app:proofs}.

\subsection{Explicit Low-Rank Representation of $\frac{\partial \mathcal{L}} {\partial \mW}$} \label{sec:low_rank}

We first show that the weight gradients $\frac{\partial \mathcal{L}} {\partial \mW}$ can be written as follows:

\begin{restatable}{theorem}{dLdWdLdZXT} \label{theorem:dLdW_dLdZ.XT}
	The network's gradient w.r.t. the weights $\mW$ can be represented as the matrix product:
	\begin{equation}\label{eq:decomp}
		\frac{\partial \mathcal{L}} {\partial \mW} = \frac{\partial \mathcal{L}} {\partial \mZ}  \mX^T.
	\end{equation}
\end{restatable}

For batch sizes $b \leq n, m$, the dimensionalities of $\frac{\partial \mathcal{L}} {\partial \mZ} \in \mathbb{R}^{m \times b}$ and $\mX \in \mathbb{R}^{n \times b}$ in \cref{eq:decomp} directly yield that the rank of $\frac{\partial \mathcal{L}} {\partial \mW}$ is at most $b$. This confirms the observations of \citet{cocktail} and shows that $\mX$ and $\frac{\partial \mathcal{L}} {\partial \mZ}$ correspond to a specific low-rank decomposition of $\frac{\partial \mathcal{L}} {\partial \mW}$.

To actually find this decomposition and thus recover $\mX$, we first consider an arbitrary decomposition of the form $\frac{\partial \mathcal{L}} {\partial \mW}= \mL  \mR$, where $\mL \in \mathbb{R}^{m \times b}$ and $\mR \in \mathbb{R}^{b \times n}$ are of maximal rank. We chose the decomposition obtained via the reduced SVD decomposition of $\frac{\partial \mathcal{L}} {\partial \mW} = \mU  \mS  \mV$ by setting $\mL = \mU  \mS ^ {\frac{1}{2}}$ and $\mR = \mS ^ {\frac{1}{2}}  \mV$, where $\mU \in \mathbb{R}^{n \times b}$, $\mS \in \mathbb{R}^{b \times b}$ and $\mV \in \mathbb{R}^{b \times n}$. We now show that there exists an unique \nameQ $\mQ$ recovering $\mX$ and $\frac{\partial \mathcal{L}} {\partial \mZ}$ from $\mL$ and $\mR$:

\begin{restatable}{theorem}{q} \label{theorem:q}
	If the gradient $\frac{\partial \mathcal{L}} {\partial \mZ}$ and the input matrix $\mX$ are of full-rank and $b \leq n, m$, then there exists an unique matrix $\mQ \in \mathbb{R}^{b \times b}$ of full-rank s.t. $\frac{\partial \mathcal{L}} {\partial \mZ} = \mL  \mQ $ and $\mX^T = \mQ^{-1}  \mR$.
\end{restatable}

\cref{theorem:q} is a direct application of \cref{lemma:decomp} shown in \cref{app:proofs}, a general linear algebra result stating that under most circumstances different low-rank matrix decompositions can be transformed into each other via an unique invertible matrix. Crucially, this implies that recovering the input $\mX$ and the gradient $\frac{\partial \mathcal{L}} {\partial \mZ}$ matrices is equivalent to obtaining the unique \nameQ $\mQ$. Next, we show how the ReLU-induced sparsity patterns in $\frac{\partial \mathcal{L}} {\partial \mZ}$ or $\mX$ can be leveraged to recover $\mQ$ exactly.

\subsection{ReLU-Induced Sparsity} \label{sec:sparsity}
ReLU activation layers can induce sparsity both in the gradient $\frac{\partial \mathcal{L}} {\partial \mZ}$ (if the ReLU activation succeeds the considered linear layer) or in the input (if the ReLU activation precedes the linear layer).

\paragraph{Gradien Sparsity} If a ReLU activation succeeds the linear layer, i.e.,  $\mY = \relu(\mZ)$, we have $\frac{\partial \mathcal{L}} {\partial \mZ} = \frac{\partial \mathcal{L}} {\partial \mY} \odot \mathds{1}_{[\mZ>0]}$, where $\odot$ is the elementwise multiplication and $\mathds{1}_{[\mZ>0]}$ is a matrix of $0$s and $1$s with each entry indicating if the corresponding entry in $\mZ$ is positive. At initialization, roughly half of the entries in $\mZ$ are positive, making $\frac{\partial \mathcal{L}} {\partial \mZ}$ sparse with $\sim0.5$ of the entries $=0$. 

\paragraph{Input Sparsity} ReLUs also introduce sparsity if the linear layer in question is preceded by a ReLU activation. Here, $\mX = \relu(\tilde{\mZ})$ will again be sparse with $\sim0.5$ of the entries $=0$ at initialization.

Note that for all but the first and the last layer of a fully connected network, we have sparsity in both, $\mX$ and $\frac{\partial \mathcal{L}} {\partial \mZ}$. 
Due to the symmetry of their formulas in \cref{theorem:q}, our method can be applied in all three arising sparsity settings. 
In the remainder of this work, we assume w.l.o.g. that only $\frac{\partial \mathcal{L}} {\partial \mZ}$ is sparse, corresponding to the first layer of a fully connected network. We now describe how to leverage this sparsity to compute the \nameQ $\mQ$ and thus recover the input batch $\mX$. 

\subsection{Breaking Aggregation through Sparsity} \label{sec:breaking_aggregation}
Our exact recovery algorithm for the \nameQ $\mQ$ is based on the following insight:

If we can construct two submatrices $\mA \in \mathbb{R}^{b - 1 \times b}$ and $\mL_A \in \mathbb{R}^{b-1\times b}$ by choosing $b-1$ rows with the same indices from $\frac{\partial \mathcal{L}}{\partial \mZ}$ and $\mL$, respectively, such that $\mA$ has full rank and an all-zero $i^\text{th}$ column, then the kernel $\ker(\mL_\mA)$ of $\mL_\mA$ contains a column $\vq_i$ of $\mQ$ up to scale. We formalize this as follows:

\begin{restatable}{theorem}{dirs} \label{theorem:dirs}
	Let $\mA \in \mathbb{R}^{b-1 \times b}$ be a submatrix of $\frac{\partial \mathcal{L}}{\partial \mZ}$ s.t. its $i^\text{th}$ column is $\mathbf{0}$ for some $i \in \{1, \dots, b\}$. 
	Further, let $\frac{\partial \mathcal{L}}{\partial \mZ}$, $\mX$, and $\mA$ be of full rank and $\mQ$ be as in \cref{theorem:q}. 
	Then, there exists a full-rank submatrix $\mL_A \in \mathbb{R}^{b-1\times b}$ of $\mL$ s.t. $\Span(\vq_i) = \ker(\mL_A)$ for the $i^\text{th}$ column $\vq_i$ of $\mQ = (\vq_1 | \cdots | \vq_b)$. 
\end{restatable}

\begin{proof}
	Pick an $i \in \{1, \dots, b\}$. 
	By assumption, there exists a submatrix $\mA \in \mathbb{R}^{b-1 \times b}$ of $\frac{\partial \mathcal{L}}{\partial \mZ}$ of rank $b-1$ whose $i^\text{th}$ column is $\mathbf{0}$. 
	To construct $\mL_\mA$, we take rows from $\mL$ with indices corresponding to $\mA$'s row indices in $\frac{\partial \mathcal{L}}{\partial \mZ}$. As $\frac{\partial \mathcal{L}}{\partial \mZ}$ and $\mX$ have full rank, by \cref{theorem:q}, we know that $\tfrac{\partial \mathcal{L}}{\partial \mZ} = \mL \mQ$, and hence $\mA = \mL_\mA \mQ$. Multiplying from the right with $\ve_i$ yields $0 = \mA \ve_i = \mL_\mA \mQ \ve_i = \mL_\mA \vq_i$,
	and hence $\ker(\mL_\mA) \supseteq \Span(\vq_i)$. 
	Further, as $\rank(\mA) = b-1$ and $\rank(\mQ) = b$, we have that $\rank(\mL_\mA) = b-1$. By the rank-nullity theorem $\dim(\ker(\mL_\mA)) = 1$ and hence $\ker(\mL_\mA) = \Span(\vq_i)$.
\end{proof}

As $\frac{\partial \mathcal{L}}{\partial \mZ}$ is not known a priori, we can not simply search for such a set of rows. Instead, we have to sample submatrices $\mL_\mA$ of $\mL$ at random and then filter them using the approach discussed in \cref{sec:filter_validate}. However, we will show in \cref{sec:sampling_analysis} that we will find suitable submatrices with high probability for $b<m$ due to the sparsity of $\tfrac{\partial \mathcal{L}}{\partial \mZ}$ and the large number $\binom{m}{b-1}$ of possible submatrices. We will now discuss how to recover the scale of the columns $\vq_i$ given their unscaled directions $\overline{\vq}_i$ forming $\overline{\mQ}$.

\subsection{Obtaining $\pmb{Q}$: Recovering the Scale of columns in $\overline{\pmb{Q}}$}\label{sec:scale}

Given a set of $b$ correct directions $\overline{\mQ} = (\overline{\vq_1} | \cdots | \overline{\vq_b})$, we can recover their scale, enabling us to reconstruct $\mX$, as follows. We first represent the correctly scaled columns as $\vq_i = s_i \cdot \overline{\vq_i}$ with the unknown scale parameters $s_i\in\mathbb{R}$.
Now, recovering the scale is equivalent to computing all $s_i$. To this end, we leverage the gradient w.r.t. the bias $\frac{\partial \mathcal{L}}{\partial \vb}$:

\begin{restatable}{theorem}{bias} \label{lemma:bias}
	The gradient w.r.t. the bias $\vb$ can be written in the form
	$\frac{\partial \mathcal{L}} {\partial \vb} = \frac{\partial \mathcal{L}} {\partial \mZ} \begin{bsmallmatrix} 1 \\ \vphantom{\int\limits^x}\smash{\vdots} \\ 1 \end{bsmallmatrix}$.
\end{restatable}

Thus, the coefficients $s_i$ can be calculated as:

\begin{restatable}{theorem}{scaling} \label{theorem:scaling}
	For any left inverse $\mL^{-L}$ of $\mL$, we have
		$\begin{bsmallmatrix} s_1 \\ \vphantom{\int\limits^x}\smash{\vdots} \\ s_b \end{bsmallmatrix} = \overline{\mQ}\,^{-1} \mL^{-L} \frac{\partial \mathcal{L}} {\partial \vb} $
\end{restatable}

\cref{theorem:scaling} allows us to directly obtain the true matrix $\mQ = \overline{\mQ} \diag(s_1, \ldots, s_b)$ from the unscaled matrix $\overline{\mQ}$. We now discuss how to recover $\overline{\mQ}$ via sampling and filtering candidate directions $\overline{\vq}_i$.

	\section{Efficient Filtering and Validation of Candidates}
\label{sec:filter_validate}

In the previous section, we saw that given the correct selection of submatrices $\mL_\mA$, we can recover $\mQ$ directly. However, we do not know how to pick $\mL_\mA$ a priori. To solve this, we rely on a sampling approach: We first randomly sample submatrices $\mL_\mA$ of $\mL$ and corresponding direction candidates $\overline{\vq}'$ spanning $\ker(\mL_\mA)$. 
However, checking whether $\overline{\vq}'$ is a valid direction is not straightforward as we do not know $\frac{\partial \mathcal{L}}{\partial \mZ}$ and hence can not observe $\mA$ directly as reconstructing $\frac{\partial \mathcal{L}}{\partial \mZ} = \mL \mQ$ requires the full $\mQ$. 

To address this, we filter the majority of wrong proposals $\overline{\vq}'$ using deduplication and a sparsity-based criterion (\cref{sec:filtering}), leaving us with a set of candidate directions $\mathcal{C} = \{\overline{\vq}_j'\}_{j \in \{1, \dots, c\}}$. We then select the correct directions in $\mathcal{C}$ greedily based on a novel sparsity matching score (\cref{sec:greedy_optimization}).

\subsection{Efficient Filtering of Directions $\overline{\pmb{q}}'$}
\label{sec:filtering}

\paragraph{Filtering Mixtures via Sparsity}
It is highly likely ($p=(1-\tfrac{1}{2^{b-1}})^b$) that a random submatrix of $\mL$ will not correspond to an $\mA$ with any $\mathbf{0}$ column. We filter these directions by leveraging the following insight. The kernel of such submatrices is spanned by a linear combination $\overline{\vq}' = \sum_i \alpha_i \overline{\vq}_i$. Thus $\mL \overline{\vq}'$ will be a linear combination of sparse columns of $\frac{\partial \mathcal{\mL}}{\partial \mZ}$. As this sparsity structure is random, linear combinations will have much lower sparsity with high probability. 
We thus discard all candidates $\overline{\vq}'$ with sparsity of $\mL \overline{\vq}'$ below a threshold $\tau$, chosen to make the probability of falsely rejecting a correct direction $p_{fr}(\tau,m) = \frac{1}{2^m} \sum_{i=0}^{\lfloor m \cdot \tau \rfloor} \binom{m}{i}$, obtained from the cumulative distribution function of the binomial distribution, small. For example for $m=400$ and $p_{fr}(\tau,m) < 10^{-5}$, we have $\tau = 0.395$.
We obtain the candidate pool $\mathcal{C} = \{\overline{\vq}_j'\}_{j \in \{1, \dots, c\}}$ from all samples that were not filered this way.

\paragraph{Filtering Duplicates}
As it is highly likely to have multiple full-rank submatrices $\mA$, whose $i^\text{th}$ column is $\mathbf{0}$, we expect to sample the same proposal $\overline{\vq}_i'$ multiple times. We remove these duplicates to substantially reduce our search space.

\subsection{Greedy Optimization} \label{sec:greedy_optimization}
While filtering duplicates and linear combinations significantly reduces the number $c$ of candidates, we usually still have to select a subset of $b<c$. 
Thus, we have $\binom{c}{b}$ possible $b$ sized subsets, each inducing a candidate $\mQ'$ and thus $\mX'$. 
A naive approach is to compute the gradients for all $\mX'$ and compare them to the ground truth. However, this is computationally infeasible even for moderate $c$. 

To address this, we propose a greedy two-stage procedure optimizing a novel sparsity matching score $\lambda$, which resolves the computational complexity issue above while also accurately selecting the correct batch elements and relying solely on $\frac{\partial \mathcal{L}}{\partial \mZ}'$ and $\mZ'$. As both can be computed directly via $\mQ'$, the procedure is local and does not need to backpropagate gradients. 
Next, we explain the first stage.

\vspace{-1mm}
\paragraph{Dictionary Learning \citep{spielman2012exact}}
As a first stage, we leverage a component of the algorithm proposed by \citet{spielman2012exact} for sparsely-used dictionary learning. This approach is based on the insight that the subset of column vectors $\mathcal{B}=\{\overline{\vq}_i'\}_{i=1}^b$, yielding the sparsest full-rank gradient matrix $\frac{\partial L}{\partial \mZ}$ is often correct. As the scaling of $\overline{\vq}_i'$ does not change the sparsity of the resulting $\frac{\partial L}{\partial \mZ}$, we can construct the subset $\mathcal{B}$ by greedily collecting the $b$ directions $\overline{\vq}_i'$ with the highest corresponding sparsity that still increase the rank of $\mathcal{B}$. While this method typically recovers most directions $\overline{\vq}_i$, it often misses directions whose gradients $\frac{\partial L}{\partial \vz_i}$ are less sparse by chance.

\vspace{-1mm}
\paragraph{Sparsity Matching}
We alleviate this issue by introducing a second stage to the algorithm where we greedily optimize a novel correctness measure based solely on the gradients of the linear layer, which we call the sparsity matching coefficient $\lambda$.

\begin{definition} \label{def:gamma}
    Let $\lambda_{-}$ be the number of non-positive entries in $\mZ$ whose corresponding entries in $\frac{\partial \mathcal{L}}{\partial \mZ}$ are $0$. Similarly, let $\lambda_{+}$ be the number of positive entries in $\mZ$ whose corresponding entries in $\frac{\partial \mathcal{L}}{\partial \mZ}$ are not $0$. We call their normalized sum the \emph{sparsity matching coefficient} $\lambda$:
    \begin{equation*}
        \lambda = \frac{\lambda_{-} + \lambda_{+}}{m\cdot b}.
    \end{equation*}
    \vspace{-5mm}
\end{definition}

Intuitively, this describes how well the pre-activation values $\mZ$ match the sparsity pattern of the gradients $\frac{\partial \mathcal{L}}{\partial \mZ}$ induced by the ReLU layer (See \cref{sec:sparsity}). While this sparsity matching coefficient $\lambda$ can take values between $0$ and $1$, it is exactly $\lambda=1$ for the correct $\mX$, if the gradient $\frac{\partial \mathcal{L}}{\partial \mY}$ w.r.t. the ReLU output is dense, which is usually the case. We note that $\lambda$ can be computed efficiently for arbitrary full rank matrix $\overline{\mQ}'$ by computing $\frac{\partial \mathcal{L}}{\partial \mZ}' = \mL \mQ'$ and $\mZ' = \mW  \mX' + (\vb | \dots | \vb)$ for $\mX'^\top = \mQ'^{-1} \mR$.

To optimize $\lambda$, we initialize $\overline{\mQ}'$ with the result of the greedy algorithm in \citet{spielman2012exact}, and then greedily swap the pair of vectors $\overline{\vq}'_i$ improving $\lambda$ the most, while keeping the rank, until convergence.

	\vspace{-1mm}
\section{Final Algorithm and Complexity Analysis} \label{sec:algorithm}
\vspace{-1mm}
In this section, we first present our final algorithm \tool (\cref{sec:final_algorithm}) and then analyse its expected complexity and failure probability (\cref{sec:sampling_analysis}).

	\begin{wrapfigure}{R}{0.61\textwidth}
	\begin{minipage}{0.61\textwidth}
		\vspace{-2.8em}
		\begin{algorithm}[H]
			\caption{\tool}
			\label{alg:gradient_dismantle}
			\begin{algorithmic}[1]
				\Function{\tool{}}{ m, n, $\mW$, $\vb$, $\frac{\partial \mathcal{L}}{\partial\mW}$, $\frac{\partial \mathcal{L}}{\partial\vb}$}
					\State $\mL, \mR, b \leftarrow \;$\textsc{LowRankDecompose}\xspace$(\frac{\partial \mathcal{L}}{\partial\mW})$ \label{alg_l:decomp}
					\For{$i=1$ {\bfseries to} $N$}
						\State Sample a submatrix $\mL_\mA \in \mathbb{R}^{b-1 \times b}$ of $\mL$ 		\label{alg_l:sample} 
						\State $\overline{\vq}_i' \leftarrow \ker(\mL_\mA)$								\label{alg_l:kern}
						\If{$\sparsity(\mL\overline{\vq}_i') \geq \tau * m$ {\bfseries and } $\overline{\vq}_i' \notin \mathcal{C}$} 	\label{alg_l:sparsity_filter}
							\State $\mathcal{C}\leftarrow \mathcal{C} \cup \{\overline{\vq}_i'\}$			\label{alg_l:candidate_pool}
							\State $\lambda, \mX' \leftarrow\;$\greedyfilter$(\mL, \mR, \mW, \vb, \frac{\partial \mathcal{L}}{\partial\vb}, \mathcal{C})$\label{alg_l:early_check}
							\If{$\lambda = 1$} \label{alg_l:early_stop}
								\State {\bfseries return} $\mX'$
							\EndIf
						\EndIf
					\EndFor
					\State $\lambda, \mX' \leftarrow\;$\greedyfilter$(\mathcal{C})$ \label{alg_l:final_check}
					\State \Return $\mX'$
				\EndFunction
			\end{algorithmic}
		\end{algorithm}
		\vspace{-3em}
	\end{minipage}
\end{wrapfigure}

\subsection{Final Algorithm}\label{sec:final_algorithm}
\vspace{-1mm}
We formalize our gradient inversion attack \tool in \cref{alg:gradient_dismantle} and outline it below.
First, we compute the low-rank decomposition $\frac{\partial \mathcal{L}}{\partial \mW} = \mL \mR$ of the weight gradient $\frac{\partial \mathcal{L}}{\partial \mW}$ via reduced SVD, allowing us to recover the batch size $b$ as the rank of $\frac{\partial \mathcal{L}}{\partial \mW}$ (\cref{alg_l:decomp}).
We now sample (at most $N$) submatrices $\mL_\mA$ of $\mL$ and compute proposal directions $\overline{\vq}_i'$ as their kernel $\ker(\mL_\mA)$ via SVD (\crefrange{alg_l:sample}{alg_l:kern}). We note that our implementation parallelizes both sampling and SVD computation (\crefrange{alg_l:sample}{alg_l:kern}) on a GPU.
We then filter the proposal directions $\overline{\vq}_i'$ based on their sparsity (\cref{alg_l:sparsity_filter}), adding them to our candidate pool $\bc{C}$ if they haven't been recovered already and are sufficiently sparse (\cref{alg_l:candidate_pool}).
Once our candidate pool contains at least $b$ directions, we begin constructing candidate input reconstructions $\mX'$ using our two-stage greedy algorithm \textsc{GreedyFilter} (\cref{alg_l:early_check}), discussed in \cref{sec:greedy_optimization}. If this reconstruction leads to a solution with sparsity matching coefficient $\lambda=1$, we terminate early and return the corresponding solution (\cref{alg_l:early_stop}). Otherwise, we continue sampling until we have reached $N$ samples and return the best reconstruction we can obtain from the resulting candidate pool (\cref{alg_l:final_check}).
The pseudocode for \textsc{ComputeSigma} (\cref{alg:compute_sigma}) and \textsc{GreedyFilter} (\cref{alg:greedy_filter}) are shown in \cref{app:algos}.

\vspace{-1mm}
\subsection{Analysis}
\label{sec:sampling_analysis}
\vspace{-1mm}
In this section, we will analyze \tool w.r.t. the number of submatrices we \emph{expect} to sample until we have recovered all $b$ correct directions $\overline{\vq}_i$ (\cref{lemma:expected_samples}), and the probability of failing to recover all $b$ correct directions despite checking all possible submatrices of $\mL$ (\cref{lemma:failure_prob}). For an analysis of the number of submatrices we have to sample until we have recovered all $b$ correct directions $\overline{\vq}_i$ \emph{with high probability}, we point to \cref{lemma:high_prob_samples}. Further, as before, we defer all proofs also to \cref{app:proofs}. 

\paragraph{Expected Number of Required Samples}

\begin{wrapfigure}{R}{0.53\textwidth}
	\begin{minipage}{0.53\textwidth}
		\vspace{-2.6em}
		\begin{figure}[H]
			\centering
			\includegraphics[width=\linewidth]{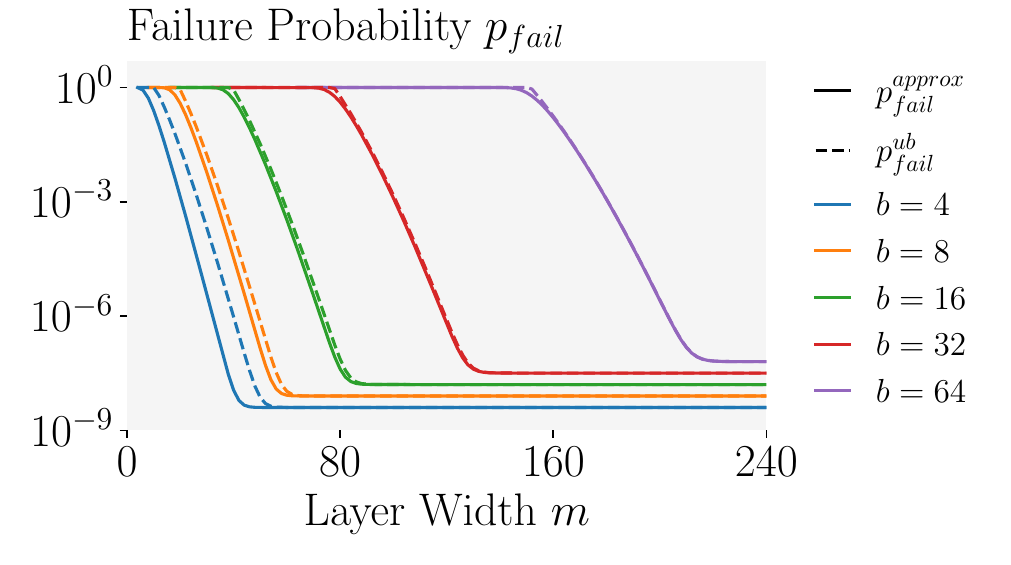}
			\vspace{-7mm}
			\caption{Visualizations of the upper bound ($p_{\text{fail}}^\text{ub}$, dashed) on and approximation of ($p_{\text{fail}}^\text{approx}$, solid) the failure probability of \tool for different batch sizes $b$ and network widths $m$ for $p_{fr} = 10^{-9}$.}
			\label{fig:failure_prob}
		\end{figure}
		\vspace{-3.4em}
	\end{minipage}
\end{wrapfigure}

To determine the expected number of required samples until we have recovered the correct $b$ direction vectors $\overline{\vq}_i$, we first compute a lower bound on the probability $q$ of sampling a submatrix which satisfies the conditions of \cref{theorem:dirs} for an arbitrary column $i$ in $\overline{\mQ}$ and then use the coupon collector problem to compute the expected number of required samples.

We can lower bound the probability of a submatrix $\mA \in \R^{b-1 \times b}$, randomly sampled as $b-1$ rows of $\tfrac{\partial \mathcal{L}}{\partial \mZ}$, having exactly one all-zero column and being full rank as follows:

\begin{restatable}{lemma}{successprob} \label{lemma:success_prob}
	Let $\mA \in \R^{b-1 \times b}$ be submatrix of the gradient $\tfrac{\partial \mathcal{L}}{\partial \mZ}$ obtained by sampling $b-1$ rows uniformly at random without replacement, where each element of $\tfrac{\partial \mathcal{L}}{\partial \mZ}$ is distributed i.i.d. as $\tfrac{\partial \mathcal{L}}{\partial \mZ_{j,k}} = \zeta |\epsilon|$ with $\epsilon \sim \mathcal{N}(\mu=0,\sigma^2>0)$ and $\zeta \sim \text{Bernoulli}(p=\tfrac{1}{2})$. We then have the probability $q$ of $\mA$ having exactly one all-zero column and being full rank lower bounded by:
	\begin{equation*}
		q \geq \frac{b}{2^{b-1}} \left(1 - (\tfrac{1}{2} + o_{b-1}(1))^{b-1}\right)
		\geq \frac{b}{2^{b-1}} (1-0.939^{b-1}).
	\end{equation*}
\end{restatable}

We can now compute the expected number of submatrices $n^*_{\text{total}}$ we have to draw until we have recovered all $b$ correct direction vectors using the Coupon Collector Problem:

\begin{restatable}{lemma}{expectedsamples} \label{lemma:expected_samples}
	Assuming i.i.d.~submatrices $\mA$ following the distribution outlined in \cref{lemma:success_prob} and using \cref{alg:gradient_dismantle}, we have the expected number of submatrices $n^*_{\text{total}}$ required to recover all $b$ correct direction vectors as:
	\begin{equation*}
		n^*_{\text{total}} = \frac{1}{q} \sum_{k=0}^{b-1} \frac{b}{b-k} = \frac{b H_b}{q} \approx \frac{1}{q} (b\log(b) + \gamma b + \tfrac{1}{2}),
	\end{equation*}
	where $H_b$ is the $b^\text{th}$ harmonic number and $\gamma \approx 0.57722$ the Euler-Mascheroni constant.
\end{restatable}

We validate this result experimentally in \cref{fig:effect_bs} where we observe excellent agreement for wide networks ($m \gg b$) and obtain, e.g., $n^*_{\text{total}} \approx 1.8 \times 10^5$ for a batch size of $b=16$.

\paragraph{Failure Probability}
We now analyze the probability of \tool failing despite considering all possible submatrices of $\mL$ and obtain:

\begin{restatable}{lemma}{failiureprobability}\label{lemma:failure_prob}
	Under the same assumptions as in \cref{lemma:success_prob}, we have an upper bound on the failure probability $p_{\text{fail}}^\text{ub}$ of \cref{alg:gradient_dismantle} even when sampling exhaustively as:
	\begin{equation*}
		p_{\text{fail}}^\text{ub} \leq b \left(1 - \sum_{k=b-1}^{m} \binom{m}{k} \frac{1}{2^m} \left(1 - 0.939^{(b-1)\binom{k}{b-1}}\right)\right) + 1-(1-p_{fr})^b,
	\end{equation*}
	where $p_{fr}$ is the probability of falsely rejecting a correct direction $\overline{\pmb{q}}'$ via our sparsity filter (\cref{sec:filtering}).
\end{restatable}

If we assume the full-rankness of submatrices $\mA$ to i) occur with probability $1 - (\frac{1}{2} - o_{b-1}(1))^{b-1}$ for $o_{b-1}(1) \approx  0$ (true for large $b$ \citep{tikhomirov2020singularity}) and ii) be independent between submatrices, we instead obtain:
\begin{equation*}
	p_{\text{fail}}^\text{approx} \approx 1 - \left(\sum_{k=b-1}^{m} \binom{m}{k} \frac{1}{2^m} \left(1 - 0.5^{(b-1)\binom{k}{b-1}}\right)\right)^b + 1-(1-p_{fr})^b.
\end{equation*}
We illustrate this bound in \cref{fig:failure_prob} and empirically validate this bound in \cref{fig:failure_prob_val} and observe the true failure probability to lie between $p_{\text{fail}}^\text{approx}$ and $p_{\text{fail}}^\text{ub}$.

	\section{Empirical Evaluation}\label{sec:eval}
\begin{wrapfigure}{R}{0.37\textwidth}
	\begin{minipage}{0.37\textwidth}
		\vspace{-2em}
		\begin{table}[H]
			\centering
			\vspace{-3.5mm}
			\caption{Comparison to prior work in the image domain.}
			\label{tab:comparison}
			\centering
			\vspace{1mm}
			\resizebox*{\linewidth}{!}{
				\begin{tabular}{@{}
						lrr@{}}
					\toprule
					Method & PSNR $\uparrow$ & Time/Batch \\
					\midrule
					CI-Net~\citep{zhang2023generative} Sigmoid & $38.0$  & 1.6 hrs\\
					CI-Net~\citep{zhang2023generative} ReLU & $15.6$ & 1.6 hrs \\
					\citet{geiping} & $19.6$  & 18.0 min \\
					\tool(Ours) & \textbf{124.2} & \textbf{2.0 min}\\
					\bottomrule		
				\end{tabular}
			}
		\end{table}
		\vspace{-2.4em}
	\end{minipage}%
\end{wrapfigure}
In this section, we empirically evaluate the effectiveness of \tool on \mnist~\citep{lecun2010mnist}, \cifar~\citep{krizhevsky2009learning}, \TIN~\citep{Le2015TinyIV}, and \ImgNet~\citep{imagenet} across a wide range of settings. In addition to the reconstruction quality metrics PSNR and LPIPS, commonly used to evaluate gradient inversion attacks, we report accuracy as the portion of batches for which we recovered the batch up to numerical errors and the number of sampled submatrices (number of iterations).

\paragraph{Experimental Setup}
For all experiments, we use our highly parallelized PyTorch~\citep{PaszkeGMLBCKLGA19} GPU implementation of \tool.   
Unless stated otherwise, we run all experiments on \cifar batches of size $b=20$ using a $6$ layer ReLU-activated FCNN with width $m=200$ and set $\tau$ to achieve a false rejection rate of $p_{fr}\leq 10^{-5}$. We supply ground truth labels to all methods \emph{except} \tool{}.

\subsection{Comparison to Prior Work}
In \cref{tab:comparison}, we compare \tool against prior gradient inversion attacks from the image domain on the \ImgNet dataset rescaled to $256\times256$ resolution. In particular, we compare to \citet{geiping}\footnote{We use so-called "modern" version of the attack from \url{https://github.com/JonasGeiping/breaching}}, as well as, the recent CI-Net~\citep{zhang2023generative}. As CI-Net only considers networks with the less common Sigmoid activations, we report its performance on both ReLU and Sigmoid versions of our network.

\begin{wrapfigure}{R}{0.55\textwidth}
	\begin{minipage}{0.55\textwidth}
		\vspace{-2em}
		\begin{table}[H]
			\centering
			\vspace{-0mm}
			\caption{Results vs prior work in the tabular domain.}
			\label{tab:tabular}
			\centering
			\vspace{1mm}
			\resizebox*{\linewidth}{!}{
				\begin{tabular}{@{}
						lrrr@{}}
					\toprule
					Method & Discr Acc (\%) $\uparrow$ & Cont. MAE $\downarrow$ & Time/Batch \\
					\midrule
					Tableak~\citep{tableak} & $97$ & 4922.7 & 2.6 min \\
					\tool(Ours) & \textbf{100}  & \textbf{20.4} & \textbf{0.4 min} \\
					\bottomrule		
				\end{tabular}
			}
		\end{table}
		\vspace{-2.2em}
	\end{minipage}%
\end{wrapfigure}
We observe that while CI-Net obtains very good reconstructions with the Sigmoid network (PSNR of 38), \tool still achieves a much higher PSNR (124) as it is exact. Further, for the more common ReLU activations, the performance of CI-Net drops significantly to a PSNR $<16$ compared to $19.6$ for \citet{geiping}. Additionally, \tool is much faster compared to both \citet{geiping} and CI-Net, taking $~10\times$ and $~100\times$ less time, respectively. Finally, we want to emphasize that both prior works rely on strong prior knowledge, including label information and knowledge of the structure of images, whereas we assume \emph{no information at all} about the data distribution and still achieve much better results in only a fraction of the time taken.

To confirm the versatility of \tool, we compare it to the SoTA attack in the tabular domain, Tableak~\citep{tableak}, in \cref{tab:tabular}. We see that due to the exact nature of our attack, we recover both continuos and discrete features better on the ADULT dataset~\citep{adult} with $b=16$, while still being $6\times$ faster.

\subsection{Main Results}\label{sec:main_result}
\begin{wrapfigure}{R}{0.64\textwidth}
	\begin{minipage}{0.64\textwidth}
		\vspace{-4.7em}
		\begin{table}[H]
			\centering
			\vspace{-4mm}
			\caption{Reconstruction quality across 100 batches.}
			\label{tab:main_results}
			\centering
			\vspace{1mm}
			\resizebox*{\linewidth}{!}{
				\begin{tabular}{@{}
						lrrrr@{}}
					\toprule
					Dataset & PSNR $\uparrow$ & LPIPS $\downarrow$ & Acc  (\%) $\uparrow$ & Time/Batch \\
					\midrule
					\mnist & $\;\,99.1$  & NaN & \textbf{99} & 2.6 min\\
					\cifar & $106.6$  & $\;\,1.16 \!\times\! 10^{-5}$ & \textbf{99} & 1.7 min \\
					\TIN   & $110.7$  & $1.62 \!\times\! 10^{-4}$ & \textbf{99} & \textbf{1.4 min} \\
					\ImgNet $224\times 224$ & $125.4$ & $1.05 \!\times\! 10^{-5}$ & \textbf{99} & 2.1 min\\
					\ImgNet $720\times 720$ & \textbf{125.6} & $\mathbf{8.08 \!\times\! 10^{-11}}$ & \textbf{99} & 2.6 min\\
					\bottomrule
				\end{tabular}
			}
		\end{table}
		\vspace{-2.4em}
	\end{minipage}%
\end{wrapfigure}
We evaluate \tool on  \mnist, \cifar, \TIN and \ImgNet at two different resolutions, reporting results in \cref{tab:main_results}. Across datasets, \tool can reconstruct almost all batches perfectly, achieving PSNRs of $100$ and above even at a batch size of $b=20$ for images as large as $720\times720$ in $<3$ minutes. We provide additional results on heterogeneous data and trained networks in \cref{app:experiments}, as well as, on the FedAvg protocol in \cref{app:fedavg}.

\begin{wrapfigure}{R}{0.46\textwidth}
	\begin{minipage}{0.45\textwidth}
		\vspace{-9mm}
		\begin{figure}[H]
			\centering
			\includegraphics[width=\linewidth]{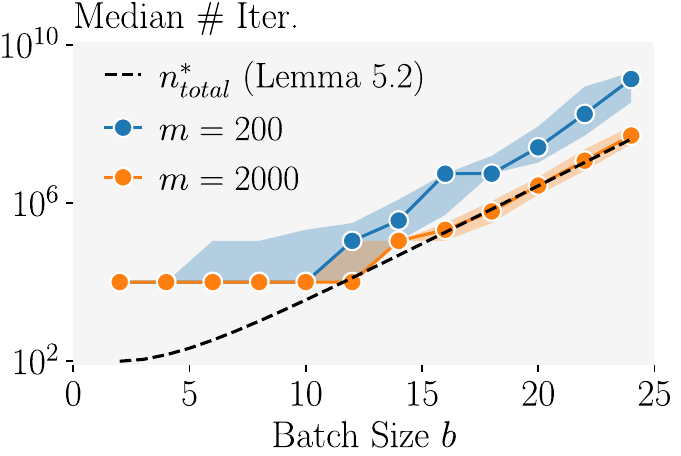}
			\vspace{-1.5em}
			\caption{Effect of batch size $b$ on the number of required submatrices. Expectation from \cref*{lemma:expected_samples} dashed and median (10\th to 90\th percentile shaded) depending on network width $m$ solid. We always evaluate $10^4$ submatrices in parallel, explaining the plateau.}
			\label{fig:effect_bs}
		\end{figure}
		\vspace{-2.5em}
	\end{minipage}%
\end{wrapfigure}

\paragraph{Effect of Batch Size $b$}

We evaluate the effect of batch size $b$ on accuracy and the required number of iterations $n^*_{\text{total}}$ for a wide ($m=2000$) and narrow ($m=200$) network. While $n^*_{\text{total}}$ increases exponentially with $b$, for both networks, the narrower network requires about $20$ times more iterations than the wider network (see \cref{fig:effect_bs}). While trends for the wider network ($m \gg b $) are perfectly described by our theoretical results in \cref{sec:sampling_analysis}, some independence assumptions are violated for the narrower network, explaining the larger number of required iterations. 
While we can recover all batches perfectly for the wider network, we see a sharp drop in accuracy from $99\%$ at $b=20$ to $63\%$ at $b=24$ (see \cref{fig:effect_2stage}) for the narrower network. This is due to increasingly more batches requiring more than the $N=2\times 10^9$ submatrices we sample at most.
\paragraph{Effect of Network Architecture}
\begin{wrapfigure}{R}{0.59\textwidth}
	\centering
	\vspace{-7mm}
	\begin{minipage}[b]{0.59\textwidth}
		\begin{subfigure}{.49\linewidth}
			\centering
			\includegraphics[width=1.05\linewidth]{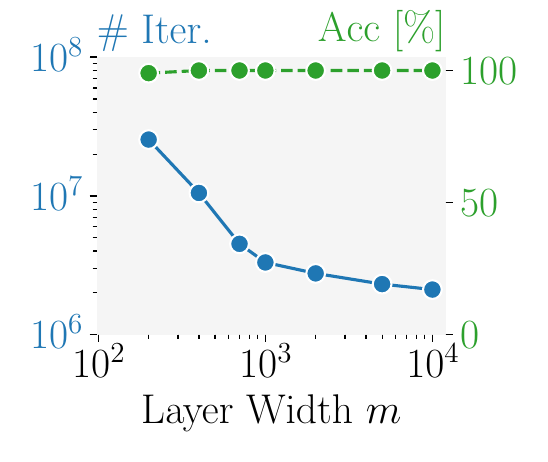}
			\vspace{-6mm}
		\end{subfigure}
		\hfil
		\begin{subfigure}{.49\linewidth}
			\centering
			\includegraphics[width=1.05\linewidth]{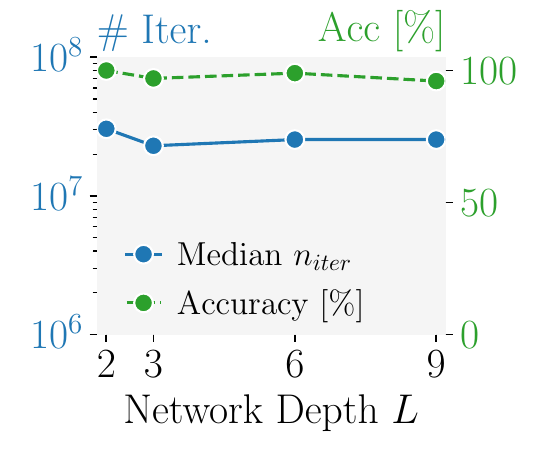}
			\vspace{-6mm}
		\end{subfigure}
		\vspace{-1mm}
		\caption{Accuracy (green) and number of median iterations (blue) for different network widths $m$ at $L=6$ (left) and depths $L$ at $m=200$ (right).}
		\label{fig:eval_architecture}
	\end{minipage}
	\vspace{-12mm}
\end{wrapfigure}
We visualize the performance of \tool across different network widths and depths in \cref*{fig:eval_architecture}. We observe that while accuracy is independent of both (given sufficient width $m \gg b$), the number of required iterations reduces with increasing width $m$. We provide further ablations on the effect of our two-stage filtering in \cref{app:twostage} and DPSGD noise in \cref{app:dp}.

\paragraph{Effect of Layer Depth}
\begin{wrapfigure}{R}{0.38\textwidth}
	\begin{minipage}{0.38\textwidth}
		\vspace{-1.0em}
		\begin{table}[H]
			\centering
			\vspace{-7.5mm}
			\caption{Effect of the attacked layer's depth $l$ ($1\leq l\leq 6$) on reconstruction time and quality for 100 \textsc{TinyImageNet} batches of size $b=20$.}
			\label{tab:layers}
			\centering
			\vspace{1mm}
			\resizebox*{\linewidth}{!}{
				\begin{tabular}{@{}l crrrr@{}} \toprule
					$l$ & MAE $\downarrow$ & Acc(\%) $\uparrow$ & Time/Batch\\ 
					\midrule
					1 & $\mathbf{1.06 \!\times\! 10^{-6}}$ & \textbf{100} & 2.3 min \\
					2 & $1.33 \!\times\! 10^{-6}$ & \textbf{100} & \textbf{2.2 min} \\
					3 & $1.67 \!\times\! 10^{-6}$ & \textbf{100} & 5.6 min \\
					4 & $2.80 \!\times\! 10^{-6}$ & $99$ & 19 min \\	
					5 & $3.04 \!\times\! 10^{-6}$ & $83$ & 70 min \\
					\bottomrule
				\end{tabular}
			}
		\vspace{-4mm}
		\end{table}
	\end{minipage}%
\end{wrapfigure}

Our experiments so far focused on recovering inputs to the first layer of FCNNs. However, \tool's capabilities extend beyond this, as highlighted in \cref{sec:sparsity}. To demonstrate this, we use \tool to reconstruct the inputs to all FC layers followed by a ReLU activation in a 6-layer FCNN with a width of $m=400$ at initialization. 

The results, presented in \cref{tab:layers}, show that \tool successfully recovers the inputs to all layers almost perfectly. However, attacking later layers is more computationally expensive. Specifically, the runtime for $l=5$ increases to 70 minutes/batch resulting in $17$ batches that timed-out. This increased computational cost is due to the initialization of the network, which causes the outputs of later layers to be dominated by their bias terms with their inputs being almost irrelevant. This issue is mitigated after a few training steps, as weights and biases adjust to better reflect the relationships between inputs and outputs. We find that after $5000$ gradient steps the time per batch reduces to $< 1$ min at an accuracy of $>95\%$ for layer $l=5$. 

\subsection{Scaling \tool via Optimization-based Attacks}\label{sec:comb}
\begin{wrapfigure}{R}{0.51\textwidth}
	\begin{minipage}{0.51\textwidth}
		\vspace{-1.0em}
		\begin{table}[H]
			\centering
			\vspace{-8mm}
			\caption{Comparison between the reconstruction quality of \citet{geiping} and a version of \tool{} that uses \citet{geiping} to speed up its search procedure evaluated on 10 \textsc{TinyImageNet} batches.}
			\label{table:comb}
			\centering
			\vspace{1mm}
			\resizebox*{\linewidth}{!}{
				\begin{tabular}{@{}l rrrrr@{}} \toprule
					Method & $b$ & $m$ &  Acc(\%) $\uparrow$ & PSNR $\uparrow$ \\
					\midrule
					\citet{geiping} & 50 & 400 & - & $26.5$ \\
					\tool+ \citet{geiping}     & 50 & 400 & 100 & \textbf{124.5} \\
					\midrule
					\citet{geiping} & 100 & 2000 & - & $32.8$ \\
					\tool+ \citet{geiping} & 100 & 2000 & $60$ & \textbf{81.5} \\ %
					\bottomrule
				\end{tabular}
			}
		\end{table}
		\vspace{-3em}
	\end{minipage}%
\end{wrapfigure}
As we prove theoretically in \cref{sec:sampling_analysis} and verify practically in \cref{app:failure_prob}, in the common regime where the batch size $b$ is much smaller than dimensions of the attacked linear layer w.h.p. the input information is losslessly represented in the client gradient. However, in practice for $b>25$ the exponential sampling complexity of \tool{} becomes a bottleneck that prevents the recovery of the input (see \cref{fig:effect_bs}). 

In this section, we propose a method for alleviating the exponential sampling complexity by combining \tool{} with an approximate reconstruction method to get a prior on which submatrices $\mL_A$ satisfy the conditions of \cref{theorem:dirs}, i.e., have corresponding matrices $\mA$ containing a 0-column. To this end, we first obtain an estimate of the client pre-activation values $\widetilde{\mZ}$ based on the approximate input reconstructions from \citet{geiping}. As large negative pre-activation values in $\widetilde{\mZ}$ are much more likely to correspond to negative pre-activation values in the true $\mZ$, and, thus, to $0$s in gradients $\tfrac{\partial \mathcal{L}}{\partial Z}$, we record the locations of the $3b$ largest negative values for each column of $\widetilde{\mZ}$. Importantly, by choosing the locations this way, we ensure that each group of $3b$ locations correspond to locations of likely 0s in \emph{same column} of $\tfrac{\partial \mathcal{L}}{\partial Z}$. Restricting the sampling of the row indices of $\mL_A$ and $\mA$ only within each group of locations, ensures that $\mL_A$ is very likely to satisfying the conditions of \cref{theorem:dirs}.

We confirm the effectiveness of this approach in a preliminary study, shown in \cref{table:comb}, that demonstrates the combined approach allows a substantial increase in the batch size \tool{} can scale to (up to 100), thus effectively eliminating its exponential complexity. The results show that the combined approach drastically improves the reconstruction quality of \citet{geiping} as well, as unlike \citet{geiping}, it achieves exact reconstruction. Importantly, we observe that even for the $4$ batches \tool{} failed to recover in \cref{table:comb}, \tool{} still reconstructs $>97$ of the $100$ directions $\overline{\vq_i}$ correctly, suggesting that future work can further improve upon our results.

\subsection{Feature Inversion in Convolutional Neural Networks}
\begin{wrapfigure}{R}{0.5\textwidth}
	\begin{minipage}{0.5\textwidth}
		\vspace{-1.0em}
		\begin{table}[H]
			\centering
			\vspace{-8.3mm}
			\caption{Comparison between the reconstructions on \textsc{VGG16} for \citet{geiping}, CPA~\citep{cocktail}, and \tool{} for 10 \textsc{ImageNet} batches ($b=16$).}
			\label{table:cocktail}
			\centering
			\vspace{0mm}
			\resizebox*{\linewidth}{!}{
				\begin{tabular}{@{}l rrrrr@{}} \toprule		
					Method & LPIPS $\downarrow$ & Feature Sim $\uparrow$ \\
					\midrule
					\citet{geiping} & $0.562$ & - \\
					CPA\citep{cocktail} + FI + \citet{geiping} & $0.388$ & $0.939$ \\
					\tool+ FI + \citet{geiping}& \textbf{0.362} & \textbf{0.984} \\ 
					\bottomrule
				\end{tabular}
			}
		\end{table}
		\vspace{-3em}
	\end{minipage}%
\end{wrapfigure}
Following the Cocktail Party Attack (CPA)~\citep{cocktail}, we experiment with using \tool{} to recover the input features to the first linear layer of a pre-trained \textsc{VGG16} convolutional network with size $25088\times4096$ for \textsc{ImageNet} batches of $b=16$ and use them in a feature inversion (FI) attack to approximately recover the client images. We show the results of our experiments, based on the CPA's code and parameters, in \cref{table:cocktail}. We see the inverted features drastically improve quality of the final reconstructions, and that \tool{} achieves almost perfect feature cosine similarity, resulting in better overall reconstruction versus CPA.

	\vspace{-1mm}
\section{Related Work}
\vspace{-1mm}
In this section, we discuss how we relate to prior work.

\vspace{-1mm}
\paragraph{Gradient Inversion Attacks}
Since gradient inversion attacks have been introduced \citep{dlg}, two settings have emerged: 
In the \emph{malicious setting}, the server does not adhere to the training protocol and can adversarially engineer network weights that maximize leaked information \citep{cah,rtf,decepticons,fishing}. 
In the strictly harder \emph{\hbc setting}, the server follows the training protocol but still aims to reconstruct client data. 
We target the \hbc setting, where prior work has either recovered the input exactly for batch sizes of $b=1$ \citep{analyticPhong,rgap}, or approximately for $b > 1$ \cite{geiping,nvidia,aaai,tableak,cocktail}. In this setting, we are the \emph{first to reconstruct inputs exactly for batch sizes $b > 1$}.

Most closely related to our work is \citet{cocktail} which leverage the low-rank structure of the gradients to frame gradient inversion as a blind source separation problem, improving their approximate reconstructions. In contrast, we derive an \emph{explicit} low-rank representation and additionally leverage gradient sparsity reconstruct inputs exactly.

Unlike a long line of prior work, we rely neither on any priors on the data distribution \citep{tableak,lamp,gupta2022recovering,ggl} nor on a reconstructed classification label \cite{geiping,idlg,aaai,nvidia,tableak}. This allows our approach to be employed in a much wider range of settings where neither is available.

\vspace{-1mm}
\paragraph{Defenses Against Gradient Inversion}
Defenses based on Differential Privacy \citep{dpsgd} add noise to the computed gradients on the client side, providing provable privacy guarantees at the cost of significantly reduced utility. 
Another line of work increases the empirical difficulty of inversion by increasing the effective batch size, by securely aggregating gradients from multiple clients \citep{secagg} or doing multiple gradient update steps locally before sharing an aggregated weight update \citep{fedavg}.
Finally, different heuristic defenses such as gradient pruning \citep{dlg} have been proposed, although their effectiveness has been questioned \citep{bayesian}.

\paragraph{Sparsely-used Dictionary learning}
Recovering the \nameQ $\mQ$ is related to the well-studied problem of sparsely-used dictionary learning. However, there the aim is to find the sparsest coefficient matrix (corresponding to our $\frac{\partial \mathcal{L}}{\partial \mZ}$) and dense dictionary ($\mQ^{-1}$) approximately encoding a signal ($\mL$). In contrast, we do not search for the sparsest solution yielding an approximate reconstruction but a solution that exactly induces consistent $\mX$ and $\frac{\partial \mathcal{L}}{\partial \mZ}$, which happens to be sparse. Sparsely-used dictionary learning is known to be NP-hard \citep{tillmann2014computational} and typically solved approximately \citep{ksvd,spielman2012exact,sun2016complete}. However, under sufficient sparsity, it can be solved exactly in polynomial time \citep{spielman2012exact}. While our $\frac{\partial \mathcal{L}}{\partial \mZ}$ are not sparse enough, we still draw inspiration from \citet{spielman2012exact} in \cref{sec:filter_validate}.

	\vspace{-1mm}
\section{Limitations}
\vspace{-1mm}
\label{sec:limitations}

We focus on recovering the inputs to fully connected layers with ReLU activations such as they occur at the beginning of fully connected networks or as aggregation layers of many other architectures. Extending our approach to other layers is an interesting direction for future work.

Further, our approach scales exponentially with batch size $b$. While \tool{}'s massive parallelizability and its ability to be combined with optimization-based attacks, as shown in \cref{sec:comb}, can partially mitigate the computational complexity, future research is still required to make reconstruction of batches of size $b>100$ practical. 
	\section{Conclusion}
\vspace{-1mm}
We propose \tool, the first algorithm permitting batches of $b>1$ elements to be recovered exactly in the \hbc setting. We demonstrate theoretically and empirically that \tool succeeds with high probability and  that our highly parallelized GPU implementation is effective across a wide range of settings, including batches of up to $25$ elements and large networks and inputs. 

We thereby demonstrate that contrary to prior belief, an exact reconstruction of batches is possible in the \hbc setting, suggesting that federated learning on ReLU networks might be inherently more susceptible than previously thought. To still protect client privacy, large effective batch sizes, obtained, e.g., via secure aggregation across a large number of clients, might prove instrumental by making reconstruction computationally intractable. 
	\newpage
	\subsubsection*{Acknowledgments}
This research was partially funded by the Ministry of Education and Science of Bulgaria (support for INSAIT, part of the Bulgarian National Roadmap for Research Infrastructure).

This work has been done as part of the EU grant ELSA (European Lighthouse on Secure and Safe AI, grant agreement no. 101070617) . Views and opinions expressed are however those of the authors only and do not necessarily reflect those of the European Union or European Commission. Neither the European Union nor the European Commission can be held responsible for them.

The work has received funding from the Swiss State Secretariat for Education, Research and Innovation (SERI).

	\message{^^JLASTBODYPAGE \thepage^^J}
	\bibliographystyle{unsrtnat}
	\bibliography{references}

\begin{thebibliography}{36}
\providecommand{\natexlab}[1]{#1}
\providecommand{\url}[1]{\texttt{#1}}
\expandafter\ifx\csname urlstyle\endcsname\relax
  \providecommand{\doi}[1]{doi: #1}\else
  \providecommand{\doi}{doi: \begingroup \urlstyle{rm}\Url}\fi

\bibitem[Geiping et~al.(2020)Geiping, Bauermeister, Dr{\"o}ge, and
  Moeller]{geiping}
Jonas Geiping, Hartmut Bauermeister, Hannah Dr{\"o}ge, and Michael Moeller.
\newblock Inverting gradients-how easy is it to break privacy in federated
  learning?
\newblock \emph{NeurIPS}, 2020.

\bibitem[McMahan et~al.(2017)McMahan, Moore, Ramage, Hampson, and
  y~Arcas]{fedsgd}
Brendan McMahan, Eider Moore, Daniel Ramage, Seth Hampson, and
  Blaise~Ag{\"{u}}era y~Arcas.
\newblock Communication-efficient learning of deep networks from decentralized
  data.
\newblock In \emph{{AISTATS}}, 2017.

\bibitem[Zhu et~al.(2019)Zhu, Liu, and Han]{dlg}
Ligeng Zhu, Zhijian Liu, and Song Han.
\newblock Deep leakage from gradients.
\newblock In \emph{NeurIPS}, 2019.

\bibitem[Dimitrov et~al.(2022)Dimitrov, Balunović, Konstantinov, and
  Vechev]{fedavg}
Dimitar~Iliev Dimitrov, Mislav Balunović, Nikola Konstantinov, and Martin
  Vechev.
\newblock Data leakage in federated averaging.
\newblock \emph{Transactions on Machine Learning Research}, 2022.
\newblock ISSN 2835-8856.
\newblock URL \url{https://openreview.net/forum?id=e7A0B99zJf}.

\bibitem[Phong et~al.(2018)Phong, Aono, Hayashi, Wang, and
  Moriai]{analyticPhong}
Le~Trieu Phong, Yoshinori Aono, Takuya Hayashi, Lihua Wang, and Shiho Moriai.
\newblock Privacy-preserving deep learning via additively homomorphic
  encryption.
\newblock \emph{{IEEE} Trans. Inf. Forensics Secur.}, \penalty0 (5), 2018.

\bibitem[Zhu and Blaschko(2021)]{rgap}
Junyi Zhu and Matthew~B. Blaschko.
\newblock {R-GAP:} recursive gradient attack on privacy.
\newblock In \emph{{ICLR}}, 2021.

\bibitem[Geng et~al.(2021)Geng, Mou, Li, Li, Beyan, Decker, and Rong]{aaai}
Jiahui Geng, Yongli Mou, Feifei Li, Qing Li, Oya Beyan, Stefan Decker, and
  Chunming Rong.
\newblock Towards general deep leakage in federated learning.
\newblock \emph{arXiv}, 2021.

\bibitem[Vero et~al.(2022)Vero, Balunović, Dimitrov, and Vechev]{tableak}
Mark Vero, Mislav Balunović, Dimitar~I. Dimitrov, and Martin~T. Vechev.
\newblock Data leakage in tabular federated learning.
\newblock \emph{ICML}, 2022.

\bibitem[Kariyappa et~al.(2023)Kariyappa, Guo, Maeng, Xiong, Suh, Qureshi, and
  Lee]{cocktail}
Sanjay Kariyappa, Chuan Guo, Kiwan Maeng, Wenjie Xiong, G~Edward Suh,
  Moinuddin~K Qureshi, and Hsien-Hsin~S Lee.
\newblock Cocktail party attack: Breaking aggregation-based privacy in
  federated learning using independent component analysis.
\newblock In \emph{International Conference on Machine Learning}, pages
  15884--15899. PMLR, 2023.

\bibitem[Spielman et~al.(2012)Spielman, Wang, and Wright]{spielman2012exact}
Daniel~A Spielman, Huan Wang, and John Wright.
\newblock Exact recovery of sparsely-used dictionaries.
\newblock In \emph{Conference on Learning Theory}, pages 37--1. JMLR Workshop
  and Conference Proceedings, 2012.

\bibitem[Tikhomirov(2020)]{tikhomirov2020singularity}
Konstantin Tikhomirov.
\newblock Singularity of random bernoulli matrices.
\newblock \emph{Annals of Mathematics}, 191\penalty0 (2):\penalty0 593--634,
  2020.

\bibitem[Zhang et~al.(2023)Zhang, Xiaoman, Sotthiwat, Xu, Liu, Zhen, and
  Liu]{zhang2023generative}
Chi Zhang, Zhang Xiaoman, Ekanut Sotthiwat, Yanyu Xu, Ping Liu, Liangli Zhen,
  and Yong Liu.
\newblock Generative gradient inversion via over-parameterized networks in
  federated learning.
\newblock In \emph{Proceedings of the IEEE/CVF International Conference on
  Computer Vision}, pages 5126--5135, 2023.

\bibitem[LeCun et~al.(2010)LeCun, Cortes, and Burges]{lecun2010mnist}
Yann LeCun, Corinna Cortes, and CJ~Burges.
\newblock Mnist handwritten digit database.
\newblock \emph{ATT Labs [Online]. Available:
  http://yann.lecun.com/exdb/mnist}, 2, 2010.

\bibitem[Krizhevsky et~al.(2009)Krizhevsky, Hinton,
  et~al.]{krizhevsky2009learning}
Alex Krizhevsky, Geoffrey Hinton, et~al.
\newblock Learning multiple layers of features from tiny images.
\newblock 2009.

\bibitem[Le and Yang(2015)]{Le2015TinyIV}
Ya~Le and Xuan~S. Yang.
\newblock Tiny imagenet visual recognition challenge.
\newblock \emph{CS 231N}, 7\penalty0 (7), 2015.

\bibitem[Deng et~al.(2009)Deng, Dong, Socher, Li, Li, and Fei-Fei]{imagenet}
Jia Deng, Wei Dong, Richard Socher, Li-Jia Li, Kai Li, and Li~Fei-Fei.
\newblock Imagenet: A large-scale hierarchical image database.
\newblock In \emph{2009 IEEE conference on computer vision and pattern
  recognition}, pages 248--255. Ieee, 2009.

\bibitem[Paszke et~al.(2019)Paszke, Gross, Massa, Lerer, Bradbury, Chanan,
  Killeen, Lin, Gimelshein, Antiga, Desmaison, K{\"{o}}pf, Yang, DeVito,
  Raison, Tejani, Chilamkurthy, Steiner, Fang, Bai, and
  Chintala]{PaszkeGMLBCKLGA19}
Adam Paszke, Sam Gross, Francisco Massa, Adam Lerer, James Bradbury, Gregory
  Chanan, Trevor Killeen, Zeming Lin, Natalia Gimelshein, Luca Antiga, Alban
  Desmaison, Andreas K{\"{o}}pf, Edward~Z. Yang, Zachary DeVito, Martin Raison,
  Alykhan Tejani, Sasank Chilamkurthy, Benoit Steiner, Lu~Fang, Junjie Bai, and
  Soumith Chintala.
\newblock Pytorch: An imperative style, high-performance deep learning library.
\newblock In Hanna~M. Wallach, Hugo Larochelle, Alina Beygelzimer, Florence
  d'Alch{\'{e}}{-}Buc, Emily~B. Fox, and Roman Garnett, editors, \emph{Advances
  in Neural Information Processing Systems 32: Annual Conference on Neural
  Information Processing Systems 2019, NeurIPS 2019, December 8-14, 2019,
  Vancouver, BC, Canada}, pages 8024--8035, 2019.
\newblock URL
  \url{https://proceedings.neurips.cc/paper/2019/hash/bdbca288fee7f92f2bfa9f7012727740-Abstract.html}.

\bibitem[Becker and Kohavi(1996)]{adult}
Barry Becker and Ronny Kohavi.
\newblock {Adult}.
\newblock UCI Machine Learning Repository, 1996.
\newblock {DOI}: https://doi.org/10.24432/C5XW20.

\bibitem[Boenisch et~al.(2021)Boenisch, Dziedzic, Schuster, Shamsabadi,
  Shumailov, and Papernot]{cah}
Franziska Boenisch, Adam Dziedzic, Roei Schuster, Ali~Shahin Shamsabadi, Ilia
  Shumailov, and Nicolas Papernot.
\newblock When the curious abandon honesty: Federated learning is not private.
\newblock \emph{arXiv}, 2021.

\bibitem[Fowl et~al.(2022{\natexlab{a}})Fowl, Geiping, Czaja, Goldblum, and
  Goldstein]{rtf}
Liam~H. Fowl, Jonas Geiping, Wojciech Czaja, Micah Goldblum, and Tom Goldstein.
\newblock Robbing the fed: Directly obtaining private data in federated
  learning with modified models.
\newblock In \emph{{ICLR}}, 2022{\natexlab{a}}.

\bibitem[Fowl et~al.(2022{\natexlab{b}})Fowl, Geiping, Reich, Wen, Czaja,
  Goldblum, and Goldstein]{decepticons}
Liam Fowl, Jonas Geiping, Steven Reich, Yuxin Wen, Wojtek Czaja, Micah
  Goldblum, and Tom Goldstein.
\newblock Decepticons: Corrupted transformers breach privacy in federated
  learning for language models.
\newblock \emph{ICLR}, 2022{\natexlab{b}}.

\bibitem[Wen et~al.(2022)Wen, Geiping, Fowl, Goldblum, and Goldstein]{fishing}
Yuxin Wen, Jonas Geiping, Liam Fowl, Micah Goldblum, and Tom Goldstein.
\newblock Fishing for user data in large-batch federated learning via gradient
  magnification.
\newblock In \emph{{ICML}}, 2022.

\bibitem[Yin et~al.(2021)Yin, Mallya, Vahdat, Alvarez, Kautz, and
  Molchanov]{nvidia}
Hongxu Yin, Arun Mallya, Arash Vahdat, Jose~M. Alvarez, Jan Kautz, and Pavlo
  Molchanov.
\newblock See through gradients: Image batch recovery via gradinversion.
\newblock In \emph{{CVPR}}, 2021.

\bibitem[Balunović et~al.(2022{\natexlab{a}})Balunović, Dimitrov, Jovanović,
  and Vechev]{lamp}
Mislav Balunović, Dimitar~I. Dimitrov, Nikola Jovanović, and Martin~T.
  Vechev.
\newblock {LAMP:} extracting text from gradients with language model priors.
\newblock In \emph{NeurIPS}, 2022{\natexlab{a}}.

\bibitem[Gupta et~al.(2022)Gupta, Huang, Zhong, Gao, Li, and
  Chen]{gupta2022recovering}
Samyak Gupta, Yangsibo Huang, Zexuan Zhong, Tianyu Gao, Kai Li, and Danqi Chen.
\newblock Recovering private text in federated learning of language models.
\newblock \emph{Advances in Neural Information Processing Systems},
  35:\penalty0 8130--8143, 2022.

\bibitem[Li et~al.(2022)Li, Zhang, Liu, and Liu]{ggl}
Zhuohang Li, Jiaxin Zhang, Luyang Liu, and Jian Liu.
\newblock Auditing privacy defenses in federated learning via generative
  gradient leakage.
\newblock In \emph{Proceedings of the IEEE/CVF Conference on Computer Vision
  and Pattern Recognition}, pages 10132--10142, 2022.

\bibitem[Zhao et~al.(2020)Zhao, Mopuri, and Bilen]{idlg}
Bo~Zhao, Konda~Reddy Mopuri, and Hakan Bilen.
\newblock idlg: Improved deep leakage from gradients.
\newblock \emph{arXiv}, 2020.

\bibitem[Abadi et~al.(2016)Abadi, Chu, Goodfellow, McMahan, Mironov, Talwar,
  and Zhang]{dpsgd}
Mart{\'{\i}}n Abadi, Andy Chu, Ian~J. Goodfellow, H.~Brendan McMahan, Ilya
  Mironov, Kunal Talwar, and Li~Zhang.
\newblock Deep learning with differential privacy.
\newblock In \emph{{CCS}}, 2016.

\bibitem[Bonawitz et~al.(2016)Bonawitz, Ivanov, Kreuter, Marcedone, McMahan,
  Patel, Ramage, Segal, and Seth]{secagg}
Kallista~A. Bonawitz, Vladimir Ivanov, Ben Kreuter, Antonio Marcedone,
  H.~Brendan McMahan, Sarvar Patel, Daniel Ramage, Aaron Segal, and Karn Seth.
\newblock Practical secure aggregation for federated learning on user-held
  data.
\newblock \emph{NIPS}, 2016.

\bibitem[Balunović et~al.(2022{\natexlab{b}})Balunović, Dimitrov, Staab, and
  Vechev]{bayesian}
Mislav Balunović, Dimitar~Iliev Dimitrov, Robin Staab, and Martin~T. Vechev.
\newblock Bayesian framework for gradient leakage.
\newblock In \emph{{ICLR}}, 2022{\natexlab{b}}.

\bibitem[Tillmann(2014)]{tillmann2014computational}
Andreas~M Tillmann.
\newblock On the computational intractability of exact and approximate
  dictionary learning.
\newblock \emph{IEEE Signal Processing Letters}, 22\penalty0 (1):\penalty0
  45--49, 2014.

\bibitem[Aharon et~al.(2006)Aharon, Elad, and Bruckstein]{ksvd}
Michal Aharon, Michael Elad, and Alfred Bruckstein.
\newblock K-svd: An algorithm for designing overcomplete dictionaries for
  sparse representation.
\newblock \emph{IEEE Transactions on signal processing}, 54\penalty0
  (11):\penalty0 4311--4322, 2006.

\bibitem[Sun et~al.(2016)Sun, Qu, and Wright]{sun2016complete}
Ju~Sun, Qing Qu, and John Wright.
\newblock Complete dictionary recovery over the sphere i: Overview and the
  geometric picture.
\newblock \emph{IEEE Transactions on Information Theory}, 63\penalty0
  (2):\penalty0 853--884, 2016.

\bibitem[Tao and Vu(2005)]{tao2005random}
Terence Tao and Van Vu.
\newblock On random pm 1 matrices: singularity and determinant.
\newblock In \emph{Proceedings of the thirty-seventh annual ACM symposium on
  Theory of computing}, pages 431--440, 2005.

\bibitem[Parameswaran(1979)]{parameswaran1979statistics}
Ravi Parameswaran.
\newblock Statistics for experimenters: an introduction to design, data
  analysis, and model building.
\newblock \emph{JMR, Journal of Marketing Research (pre-1986)}, 16\penalty0
  (000002):\penalty0 291, 1979.

\bibitem[Chen(2011)]{Chen2011}
Zac Chen.
\newblock \emph{H2 Maths Handbook}.
\newblock Educational Publishing House, 2011.

\end{thebibliography}
	
	\message{^^JLASTREFERENCESPAGE \thepage^^J}

	\ifincludeappendixx
	\clearpage
	\appendix
	\section{Broader Impact} \label{sec:impact}

In this work, we demonstrate that contrary to prior belief, an exact reconstruction of batches is possible in the \hbc setting for federated learning. 
As our work demonstrates the susceptibility of federated learning systems using ReLU networks, this work inevitably advances the capabilities of an adversary. 
Nonetheless, we believe this to be an important step in accurately assessing the risks and utilities of federated learning systems. 

To still protect client privacy, large effective batch sizes, obtained, e.g., via secure aggregation across a large number of clients, might prove instrumental by making reconstruction computationally intractable. 
As gradient information and network states can be stored practically indefinitely, our work highlights the importance of proactively protecting client privacy in federated learning not only against current but future attacks.
This underlines the importance of related work on provable privacy guarantees obtained via differential privacy.
\section{Deferred Proofs} \label{app:proofs}

\dLdWdLdZXT*

\begin{proof}
	We will use Einstein notation for this proof:
	\begin{align*}
		\frac{\partial \mathcal{L}}{\partial \mW\indices{_i^j}} 
		&= 
		\frac{\partial \mathcal{L}}{\partial \mZ\indices{_k^l}}  
		\frac{\partial \mZ\indices{_k^l}}{\partial \mW\indices{_i^j}}
		\\
		&= 
		\frac{\partial \mathcal{L}}{\partial \mZ\indices{_k^l}}  
		\frac{\partial (\mW\indices{_k^m} \mX\indices{_m^l} + b\indices{_k} \delta\indices{^l})}{\partial \mW\indices{_i^j}} 
		\\
		&= 
		\frac{\partial \mathcal{L}}{\partial \mZ\indices{_k^l}}  
		\frac{\partial \mW\indices{_k^m} \mX\indices{_m^l}}{\partial \mW\indices{_i^j}} 
		\\
		&= 
		\frac{\partial \mathcal{L}}{\partial \mZ\indices{_k^l}}  
		\frac{\partial \mW\indices{_k^m}}{\partial \mW\indices{_i^j}} \mX\indices{_m^l}
		\\
		&= 
		\frac{\partial \mathcal{L}}{\partial \mZ\indices{_k^l}}  
		\delta\indices{_k^i} \delta\indices{_j^m} \mX\indices{_m^l}
		\\
		&= 
		\frac{\partial \mathcal{L}}{\partial \mZ\indices{_i^l}}  
		\mX\indices{_j^l}
		\\
		&= 
		\frac{\partial \mathcal{L}}{\partial \mZ\indices{_i^l}}  
		(\mX^T)\indices{^l_j}. 
	\end{align*}
	We note that $\delta\indices{_k^i}$ is the Kronecker delta, that is $\delta\indices{_k^i} = 1$ if $k = i$ and 0 otherwise. Further, $\delta\indices{^l} = 1$ for all $l$.
	Hence we arrive at \cref{eq:decomp}.
\end{proof}

\begin{restatable}{lemma}{decomp} \label{lemma:decomp}
	Let $b, n, m \in \mathbb{N}$ such that $b < n, m$. Further, let $\mA, \mL \in \mathbb{R}^{m \times b}$ and $\mB, \mR \in \mathbb{R}^{b \times n}$ be matrices of maximal rank, satisfying $\mA \mB = \mL \mR$. Then there exists a unique \nameQ $\mQ \in \text{GL}_b(\R)$ s.t. $\mA = \mL  \mQ$, and $\mB = \mQ^{-1}  \mR$. 
\end{restatable} 

\begin{proof}
	As $b \leq n, m$ and the matrices $\mA \in \mathbb{R}^{m \times b}$ and $\mB \in \mathbb{R}^{b \times n}$ have full rank, we know that there exists
	\begin{itemize}
		\item a left inverse $\mA^{-L} \in \mathbb{R}^{b \times m}$ for $\mA$: $\mA^{-L} \mA = \mI_b$ and 
		\item a right inverse $\mB^{-R} \in \mathbb{R}^{n \times b}$ for $\mB$: $\mB \mB^{-R} = \mI_b$.
	\end{itemize}
	Thus, it follows from
	\begin{equation*}
		\mA^{-L} \mL \mR \mB^{-R} 
		=
		\mA^{-L} \mA \mB \mB^{-R} 
		= \mI_b,
	\end{equation*} 
	that $(\mA^{-L} \mL)^{-1} = \mR \mB^{-R}$.
	We now set $\mQ = \mR \mB^{-R}$. 

	This $\mQ$ satisfies the required properties: 
	\begin{itemize}
		\setlength\itemsep{-1.0em}
		\item $\mB = \mQ^{-1} \mR$:
		\begin{equation*}
			\mQ^{-1} \mR = \mA^{-L} \mL \mR = \mA^{-L} \mA \mB = \mB,
		\end{equation*}
		\item $\mA = \mL \mQ$:
		\begin{equation*}
			\mL \mQ = \mL \mR \mB^{-R} = \mA \mB \mB^{-R} = \mA,
		\end{equation*}
		\item Uniqueness: Assume we have $\mQ_1$ and $\mQ_2$ that satisfy $\mL \mQ_1 = \mL \mQ_2 = \mA$. As $\mL$ is of rank $b$ and $b \leq m$, there exists a left inverse $\mL^{-L}$ for $\mL$: $\mL^{-L} \mL = \mI_b$. Applying this left inverse to $\mL \mQ_1 = \mL \mQ_2$, directly yields $\mQ_1 = \mQ_2$, and hence we get uniqueness. 
	\end{itemize}
\end{proof}

\bias*

\begin{proof}
	We use again Einstein notation. 
	\begin{align*}
		\frac{\partial \mathcal{L}}{\partial \vb\indices{_i}}
		&=
		\frac{\partial \mathcal{L}}{\partial \mZ\indices{_k^l}}
			\frac{\partial \mZ\indices{_k^l}}{\partial \vb\indices{_i}}
		\\
		&=
		\frac{\partial \mathcal{L}}{\partial \mZ\indices{_k^l}}
		\frac{\partial (\mW\indices{_k^m} \mX\indices{_m^l} + \vb\indices{_k} \delta^l)}{\partial \vb\indices{_i}}
		\\
		&=
		\frac{\partial \mathcal{L}}{\partial \mZ\indices{_k^l}}
		\frac{\partial \vb\indices{_k} \delta^l}{\partial \vb\indices{_i}}
		\\
		&=
		\frac{\partial \mathcal{L}}{\partial \mZ\indices{_k^l}}
		\delta\indices{_k^i} \delta^l
		\\
		&=
		\frac{\partial \mathcal{L}}{\partial \mZ\indices{_i^l}}
		\delta^l.
	\end{align*}
	This concludes the proof. 
\end{proof}

\scaling*

\begin{proof}
	The proof is straight forward. Using \cref{lemma:bias} and \cref{theorem:q}, we know that 
	\begin{align*}
		\overline{\mQ}\,^{-1} \mL^{-L} \frac{\partial \mathcal{L}}{\partial \vb} 
		&= 
		\overline{\mQ}\,^{-1} \mL^{-L} \frac{\partial \mathcal{L}}{\partial \mZ} 
			\begin{bsmallmatrix} 
				1 \\ 
				\vphantom{\int\limits^x}\smash{\vdots} \\ 
				1 
			\end{bsmallmatrix} 
		\\
		&= 
		\overline{\mQ}\,^{-1} \mL^{-L} \mL  \mQ 
			\begin{bsmallmatrix} 
				1 \\ 
				\vphantom{\int\limits^x}\smash{\vdots} \\ 
				1 
			\end{bsmallmatrix} 
		\\
		&= 
		\overline{\mQ}\,^{-1}  \mQ 
			\begin{bsmallmatrix} 
				1 \\ 
				\vphantom{\int\limits^x}\smash{\vdots} \\ 
				1 
			\end{bsmallmatrix} 
		\\
		&= 
		\overline{\mQ}\,^{-1}  \overline{\mQ} \diag(s_1, \dots, s_b)  
			\begin{bsmallmatrix} 
				1 \\ 
				\vphantom{\int\limits^x}\smash{\vdots} \\ 
				1 
			\end{bsmallmatrix} 
		\\
		&= 
		\begin{bsmallmatrix} 
			s_1 \\ 
			\vphantom{\int\limits^x}\smash{\vdots} \\ 
			s_b 
		\end{bsmallmatrix}.
	\end{align*}
\end{proof}

\successprob*

\begin{proof}
	We have the probability of one of the $b$ columns being all zero as $\frac{b}{2^{b-1}}$ if the network has full rank, all other columns will not be all-zero.
	
	Further, we have the probability of the submatrix $\mathds{1}_{\mA_>0}$ being full rank conditioned on column $i$ being all-zero as the probability of the matrix described by remaining $b-1$ columns being non-singular. This probability is $1 - (\tfrac{1}{2} + o_{b-1}(1))^{b-1}$ \citep{tikhomirov2020singularity} where $\lim_{b\rightarrow\infty}o_{b-1}(1) = 0$, which can be lower-bounded with $1-0.939^{b-1}$ \citep{tao2005random}. We thus obtain their joint probability as their product.
\end{proof}

\expectedsamples*

\begin{proof}
	As we sample submatrices $\mA$ uniformly at random with replacement, assuming them to be i.i.d.~is well justified for the regime of $m \gg b$. 
	The the number $n$ of submatrices drawn between correct direction vectors $\vq_i$ thus follows a Geometric distribution $\P[n=k] = q (1-q)^{k-1}$ with success probability $q$ with expectation $n^* = \E[n] = \tfrac{1}{q}$. 
	As we draw correct direction vectors $\vq_i$ uniformly at random from the $b$ columns of $\overline{\mQ}$, we have the probability of drawing a new direction vector $\vq_i$ as $\frac{b-k}{b}$ for $k$ already drawn direction vectors. Again via the expectation of the Geometric distribution we obtain the expected number $c^*$ of correct direction vectors we have to draw until we have recovered all $b$ distinct ones as the solution of the Coupon Collector Problem $c^*=\sum_{k=0}^{b-1} \frac{b}{b-k} = b H_b \approx b\log(b) + \gamma b + \tfrac{1}{2}$. The proof concludes with the linearity of expectation.
\end{proof}

\paragraph{Maximum Number of Samples Required with High Probability}
We now compute the number of samples $n^p_{\text{total}}$ required to recover all $b$ correct directions with high probability $1-p$.

\begin{lemma} \label{lemma:high_prob_samples}
	In the same setting as \cref{lemma:expected_samples}, we have an upper bound  $n^p_{\text{total}}$ on the number of submatrices we need to sample until we have recovered all $b$ correct direction vectors by solving the following quadratic inequality for $n^p_{\text{total}}$
	\begin{equation*}
		\frac{p}{2} \leq \Phi\left(\frac{b \log(2 b/p^*) - n^p_{\text{total}} q}{\sqrt{n^p_{\text{total}} q(1-q)}}\right),
	\end{equation*}
	where $\Phi$ is the cumulative distribution function of the standard normal distribution and $p^* = p - 1 + (1-p_{fr})^b$.
\end{lemma}

\begin{proof}
	At a high level, bound the number of valid directions $c^p$ we need to discover until we recover all $b$ distinct ones and then the number of submatrices $n^p_{\text{total}}$ we need to sample to obtain these $c^p$ directions, each with probability $1-\tfrac{p}{2}$, before applying the union bound.

	However, we first note that with probability $1-(1-p_{fr})^b$ we will (repeatedly) reject a correct direction due to a lack of induced sparsity and thus fail irrespective of the number of samples we draw. We thus correct our failure probability budget from $p$ to $p^* = p - 1 + (1-p_{fr})^b$, using the union bound.

	We now show how to compute the upper bound on the number of correct directions $c^p$ we need to find until we have found all $b$ distinct directions. To this end, we bound the probability of not sampling the $i^\text{th}$ direction $\overline{\vq}_i$ after finding $c$ candidates as $p_{\lnot i} = (1-\frac{1}{b})^c \leq e^{-\frac{c}{b}}$. We can then bound the probability of missing any of the $b$ directions using the union bound as $p_{\lnot \text{all}} \leq \sum_{i=1}^b p_{\lnot i} = b e^{-\frac{c^p}{b}}$. We thus obtain the minimum number $c^p$ of correct directions to find all $b$ distinct ones with probability at least $\tfrac{p^*}{2}$ as $c^p \geq b \log(2 b/p^*)$.

	We can now compute the number $n^p_{\text{total}}$ of samples required to find $c^p$ submatrices satisfying the condition of \cref{theorem:dirs} for some $i$ with probability $1-\tfrac{p}{2}$. To this end, we approximate the Binomial distribution $\bc{B}(n,q)$ with the normal distribution $\bc{N}(nq, nq(1-q))$ \citep{parameswaran1979statistics}, which is generally precise if $\min(nq, n (q-1)) > 9$ \citep{Chen2011}, which holds for $b \geq 5$. We thus obtain the number of samples $n^p_{\text{total}}$ required to find $c^p$ valid directions with high probability $1-\tfrac{p^*}{2}$ by solving $\tfrac{p^*}{2} = \Phi(\frac{c^p - n^p_{\text{total}} q}{\sqrt{n^p_{\text{total}} q(1-q)}})$ for $n^p_{\text{total}}$ which boils down to a quadratic equation.

	By the union bound, we have that the total failure probability of not finding all $b$ correct directions is at most $p$.
\end{proof}

For a batch size of $b=10$ and $p=10^{-8}$, we, e.g., obtain $n \approx 4 \times 10^4$.

\failiureprobability*

\begin{proof}
	We will first compute the probability of $\frac{\partial \mathcal{L}}{\partial \mZ}$ not containing a submatrix $\mA$ satisfying the conditions of \cref{theorem:dirs} for all $i \in \{1, \dots, b\}$ and then the probability of us failing to discover it despite exhaustive sampling.

	We observe that the number $k$ of rows in $\frac{\partial \mathcal{L}}{\partial \mZ}$ with a zero $i^\text{th}$ entry is binomially distributed with success probability $\frac{1}{2}$. 
	For each $k \geq b-1$, we can construct $\binom{k}{b-1}$ submatrices $\mA$ with an all-zero $i^\text{th}$ column. 
	The probability of any such submatrix having full rank is $1 - (\frac{1}{2} - o_{b-1}(1))^{b-1} > 1-0.939^{b-1}$ \citep{tikhomirov2020singularity,tao2005random}.

	We thus have the probability of $\frac{\partial \mathcal{L}}{\partial \mZ}$ containing at least one submatrix $\mA$ with full rank and an all-zero $i^\text{th}$ column as 
	$\sum_{k=b-1}^{m} \binom{m}{k} \frac{1}{2^m} \left(1 - 0.939^{(b-1)\binom{k}{b-1}}\right)$.

	Using the union bound, we thus obtain an upper bound on the probability of $\frac{\partial \mathcal{L}}{\partial \mZ}$ not containing any submatrix $\mA$ with full rank and an all-zero $i^\text{th}$ column for all $i \in \{1, \dots, b\}$.

	To compute the probability of us failing to discover an existing submatrix despite exhaustive sampling, we first note that we have the probability $p_{fr}$ of an arbitrary column in $\frac{\partial \mathcal{L}}{\partial \mZ}$ being less sparse than our threshold $\tau$. Thus, with probability $1-(1-p_{fr})^b$ we will discard at least one correct direction due to it inducing an unusually dense column in $\frac{\partial \mathcal{L}}{\partial \mZ}$. 	

	We now obtain the overall failure probability via the union bound.
\end{proof}

\section{Deferred Algorithms} \label{app:algos}

Here, we present the Algorithms \textsc{\computesigma} and \textsc{GreedyFilter} referenced in \cref{sec:final_algorithm}. 
\begin{algorithm}[h]
	\caption{\computesigma}
	\label{alg:compute_sigma}
	\begin{algorithmic}[1]
		\Function{\computesigma}{$\mL, \mR, \mW, \vb, \frac{\partial \mathcal{L}}{\partial\vb}, \mathcal{B}$}
			\State $\mQ \leftarrow\;$ \textsc{FixScale}\xspace$(\mathcal{B}, \mL, \frac{\partial \mathcal{L}}{\partial\vb})$
			\State $\frac{\partial \mathcal{L}}{\partial\mZ} \leftarrow \mL\cdot\mQ$
			\State $\mX^T \leftarrow \mQ^{-1}\cdot \mR$
			\State $\mZ=\mW \cdot \mX + (\vb| \dots| \vb)$
			\State $\lambda_{-} \leftarrow \sum_{i,j} \mathds{1}[\mZ_{i,j} \leq 0] \cdot \mathds{1}[\frac{\partial \mathcal{L}}{\partial\mZ_{i,j}} = 0]$
			\State $\lambda_{+} \leftarrow \sum_{i,j} \mathds{1}[\mZ_{i,j} > 0] \cdot \mathds{1}[\frac{\partial \mathcal{L}}{\partial\mZ_{i,j}} \neq 0]$
			\State $\lambda \leftarrow \frac{\lambda_{-} + \lambda_{+}}{m\cdot b}$
			\State \Return  $\lambda$
		\EndFunction
	\end{algorithmic}
\end{algorithm}
\begin{algorithm}[h]
	\caption{\greedyfilter}
	\label{alg:greedy_filter}
	\begin{algorithmic}[1]
		\Function{GreedyFilter}{$\mL, \mR, \mW, \vb, \frac{\partial \mathcal{L}}{\partial\vb}, \mathcal{C}$}
			\State $\mathcal{B} \leftarrow \{\}$
			\While{rank of $\mathcal{B}$ is $B$}
				\State Select the sparsest vector $\overline{\vq}_i'$ from $\mathcal{C}\setminus\mathcal{B}$
				\State $\mathcal{B} \leftarrow \mathcal{B} \cup \{\overline{\vq}_i'\}$
				\If{$\mathcal{B}$ is {\bfseries not} of full rank}
					\State $\mathcal{B} \leftarrow \mathcal{B} \setminus \{\overline{\vq}_i'\}$
				\EndIf
				\EndWhile
			\State
			\State $\lambda \leftarrow\;$\computesigma$(\mL, \mR, \mW, \vb, \frac{\partial \mathcal{L}}{\partial\vb}, \mathcal{B})$
			\While{{\bfseries not} changed}
				\State changed $\leftarrow$ False
				\For{$(\overline{\vq}_i', \overline{\vq}_j')$ {\bfseries in} $\mathcal{B}\times(\mathcal{C} \setminus \mathcal{B})$}
					\State $\mathcal{B'} \leftarrow \mathcal{B} \setminus \{\overline{\vq}_i'\} \cup \{\overline{\vq}_j'\} $
					\State $\lambda'\leftarrow\;$\computesigma$(\mL, \mR, \mW, \vb, \frac{\partial \mathcal{L}}{\partial\vb},\mathcal{B}')$
					\If{$\lambda'>\lambda$}
						\State $\mathcal{B} \leftarrow \mathcal{B'}$
						\State $\lambda \leftarrow \lambda'$
						\State changed $\leftarrow$ True
					\EndIf
				\EndFor
			\EndWhile
			\State $\mQ \leftarrow\;$ \textsc{FixScale}\xspace$(\mathcal{B}, \mL, \frac{\partial \mathcal{L}}{\partial\vb})$
			\State $\mX^T \leftarrow \mQ^{-1}\cdot \mR$
			\State \Return$\lambda$, $\mX$
		\EndFunction
	\end{algorithmic}
\end{algorithm}

\section{Dataset Licenses} \label{app:datacopyright}

In this work, we use the commonly used \mnist~\cite{lecun2010mnist}, \cifar~\cite{krizhevsky2009learning}, \TIN~\cite{Le2015TinyIV} and \ImgNet~\cite{imagenet} image datasets. No information regarding licensing has been provided on their respective websites. Further, we use Adult tabular dataset under the Creative Commons Attribution 4.0 International (CC BY 4.0) license.

\section{Deferred Experiments} \label{app:experiments}

\subsection{Main Results with Error Bars}\label{app:main_with_error}
In this section, we provide the results from our main experiment in \cref{tab:main_results}, alongside 95\% confidence intervals.
\begin{table}[t]
	\centering
	\vspace{-4mm}
	\caption{Reconstruction quality across 100 batches.}
	\label{tab:main_results_error}
	\centering
	\vspace{1mm}
	\resizebox*{\linewidth}{!}{
		\begin{tabular}{@{}
				lrrrr@{}}
			\toprule
			Dataset & PSNR $\uparrow$ & LPIPS $\downarrow$ & Acc (\%) $\uparrow$ & Time/Batch \\
			\midrule
			\mnist & $\;\,99.1\pm 13.2$  & NaN & \textbf{99} & 2.6 min\\
			\cifar & $106.6 \pm 15.1$  & $\;\,1.16 \!\times\! 10^{-5} \pm 2.26 \!\times\! 10^{-4}$ & \textbf{99} & 1.7 min \\
			\TIN   & $110.7 \pm 12.8$  & $1.62 \!\times\! 10^{-4}  \pm 3.22 \!\times\! 10^{-3}$ & \textbf{99} & \textbf{1.4 min} \\
			\ImgNet $224\times 224$ & $125.4 \pm 11.2$ & $1.05 \!\times\! 10^{-5} \pm 9.50 \!\times\! 10^{-4}$ & \textbf{99} & 2.1 min\\
			\ImgNet $720\times 720$ & $\mathbf{125.6 \pm 8.1}$ & $\mathbf{8.08 \!\times\! 10^{-11} \pm 3.05 \!\times\! 10^{-3}}$ & \textbf{99} & 2.6 min\\
			\bottomrule
		\end{tabular}
	}
\end{table}

\subsection{Experiments on Label-Heterogeneous Data}\label{app:label_het}
In this section, we provide experiments on heterogeneous client data. In particular, we look at the extreme case where each client has data only from a single class. As label repetition makes optimization-based attacks harder \citep{geiping,nvidia,aaai}, the results presented in \cref{tab:same_label} for the TinyImageNet dataset show another advantage of our algorithm, namely, \tool
works regardless of the label distribution, providing even better reconstruction results compared to \cref{tab:main_results} for single-label batches. 
\begin{table}[h]
	\centering
	\vspace{-4mm}
	\caption{Mean reconstruction quality metrics across 100 batches for batches only containing samples from only one class in the same setting as \cref{tab:main_results}.}
	\label{tab:same_label}
	\centering
	\vspace{1mm}
	\resizebox*{0.7\linewidth}{!}{
		\begin{tabular}{@{}
				lccrcc@{}}
			\toprule
			Dataset & PSNR $\uparrow$ & SSIM $\uparrow$ & \multicolumn{1}{c}{MSE $\downarrow$} & LPIPS $\downarrow$ & Acc  (\%) $\uparrow$ \\
			\midrule
			\textsc{TinyImgNet}   & $127.7$  & $0.999717$ & $4.80 \!\times\! 10^{-6}$ & $10.36 \!\times\! 10^{-5}$ & $98$ \\
			
			\bottomrule
		\end{tabular}
	}
	\vspace{-3mm}
\end{table}

\subsection{Effectivness of our 2-Stage Greedy Algorithm} \label{app:twostage}
\begin{figure}[t]
	\centering
	\includegraphics[width=0.6\linewidth]{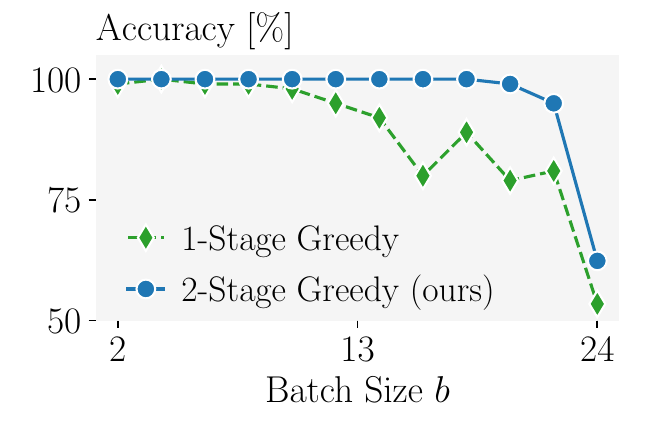}
	\vspace{-3mm}
	\caption{Effect of the second stage of our reconstruction algorithm discussed in \cref{sec:greedy_optimization}, depending on the batch size $b$.}
	\label{fig:effect_2stage}
	\vspace{-3mm}
\end{figure}

In this section, we compare reconstruction success rate (accuracy) with and without the second stage of our greedy algorithm discussed in \cref{sec:greedy_optimization} in \cref{fig:effect_2stage}. We observe that the second stage filtering becomes increasingly important for larger batch size $b$.

\subsection{Effect of Training on \tool}\label{app:train}
\begin{figure}[h]
	\centering
	\begin{subfigure}{.4\linewidth}
		\centering
		\includegraphics[width=\linewidth]{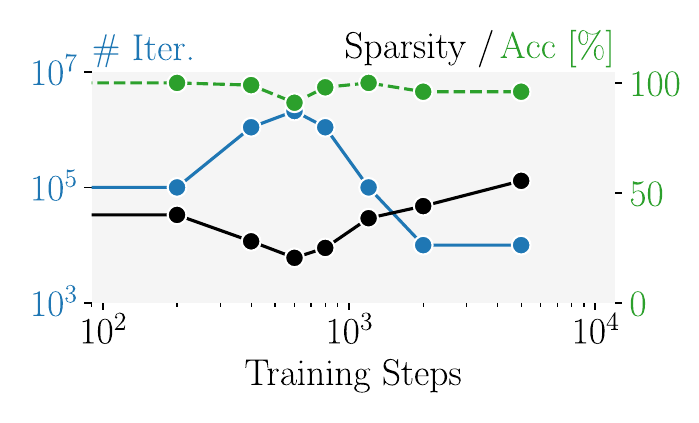}
		\vspace{-6mm}
		\subcaption{Training Set}
	\end{subfigure}
	\hfil
	\begin{subfigure}{.4\linewidth}
		\centering
		\includegraphics[width=\linewidth]{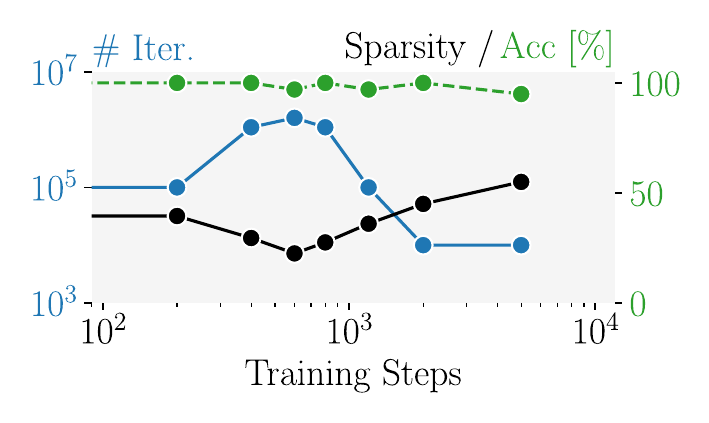}
		\vspace{-6mm}
		\subcaption{Test Set}
	\end{subfigure}
	\begin{subfigure}{.18\linewidth}
		\centering
		\hspace{-4mm}\includegraphics[width=\linewidth]{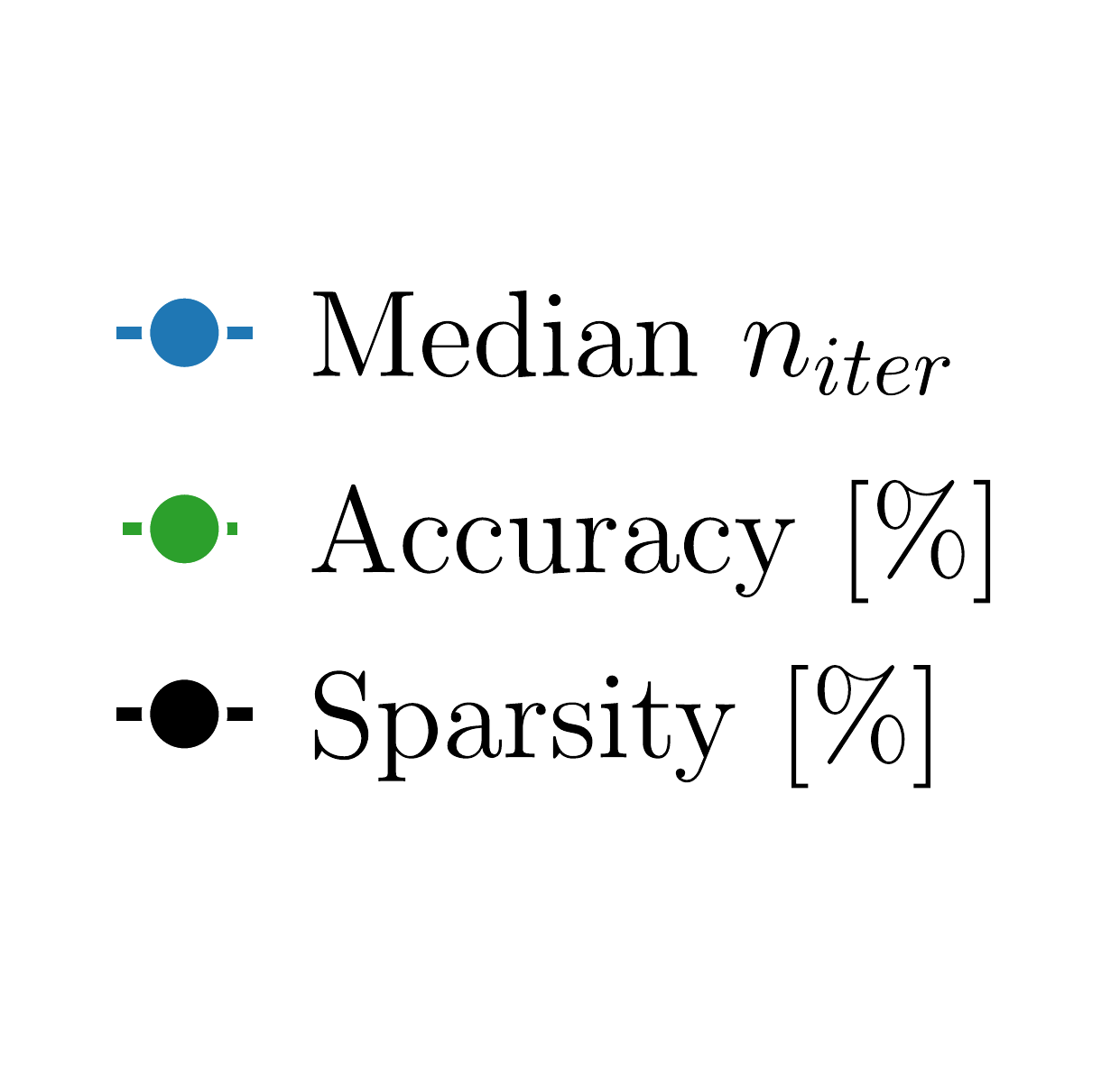}
		\vspace{10mm}
	\end{subfigure}
	\caption{Effect of training (on \mnist) on the effectiveness of \tool at a batch size of $b=10$ evaluated on the \mnist training (a) and test (b) sets.}
	\label{fig:effect_training}
	\vspace{-3mm}
\end{figure}

In this section, we demonstrate how training effects \tool's performance. To this end, we train a network on \mnist and evaluate \tool periodically during training both on the train and test datasets, visualizing results in \cref{fig:effect_training}. We observe that \tool performance is very similar between the two datasets we evaluate on. Further, we see that \tool performs very well on trained networks, with the number of required steps by the algorithm being even lower those those on untrained networks. However, if the minimum column sparsity of $\frac{\partial\bc{L}}{\partial\mZ}$ drops significantly, as is the case for the checkpoints around 1000 training steps in the illustrated run. \tool's performance drops slightly.

\subsection{Failure Probabilities}\label{app:failure_prob}
In this section, we validate experimentally our theoretical results on \tool's failure rate for several batch sizes $b$ (\cref{lemma:failure_prob}). As this requires exhaustive sampling of all $\binom{m}{b-1}$ submatrices of $\mL$ we only consider small batch sizes $b \leq 10$ and networks $m \leq 40$. We show the results in \cref{fig:failure_prob_val} where we observe that the empirical failure probability (blue) with $95\%$ Clopper-Pearson confidence bounds generally agrees with the analytical approximation (solid line) and always lies below the  analytical upper bound (dashed line). We conclude that in most settings, the number of required samples rather than complete failure is the limiting factor for \tool's performance.

\begin{figure}[h]
	\centering
\begin{subfigure}{.35\linewidth}
	\centering
	\includegraphics[width=0.9\linewidth]{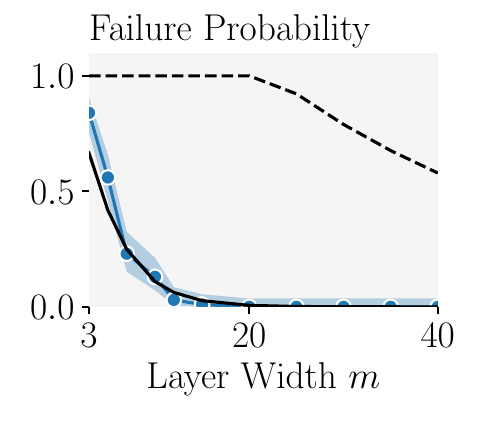}
	\vspace{-3mm}
	\subcaption{$b=2$}
\end{subfigure}
\hfil
\begin{subfigure}{.35\linewidth}
	\centering
	\includegraphics[width=0.9\linewidth]{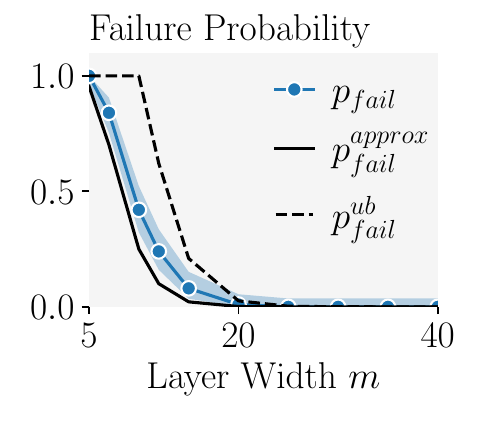}
	\vspace{-3mm}
	\subcaption{$b=4$}
\end{subfigure}
\begin{subfigure}{.35\linewidth}
	\centering
	\includegraphics[width=0.9\linewidth]{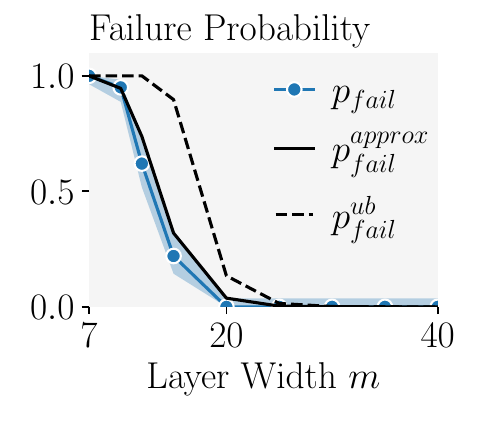}
	\vspace{-3mm}
	\subcaption{$b=6$}
\end{subfigure}
\hfil
\begin{subfigure}{.35\linewidth}
	\centering
	\includegraphics[width=0.9\linewidth]{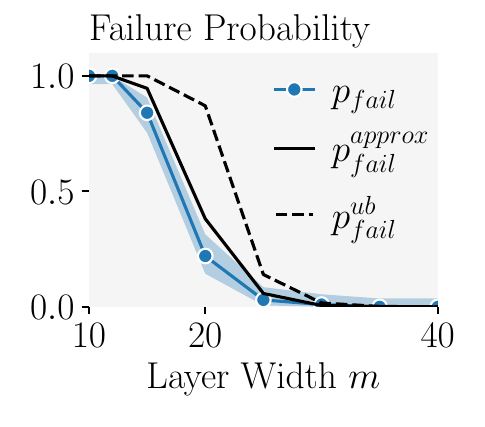}
	\vspace{-3mm}
	\subcaption{$b=8$}
\end{subfigure}
\vspace{-1mm}
\caption{Empirical failure probability (blue) with $95\%$ Clopper-Pearson confidence bounds (shaded blue) compared to the analytical upper bound (dashed line) and approximation (solid line) of the failure probability for different batch sizes $b$.}
\label{fig:failure_prob_val}
\end{figure}

\subsection{Results under DPSGD}\label{app:dp}
In this section, we show experimental results on reconstructing images from gradients defended using DP-SGD~\citep{dpsgd}. In \cref{tab:dp}, we report results on the \TIN dataset, $b=20$, with noise levels $\sigma \leq 1.0  \!\times\! 10^{-4}$ and gradient clipping that constrains the $\ell_2$ norm of the composite gradient vector, combining the gradients of all layers, to a maximum value of $C \in [1,2]$. We chose the maximum value $\sigma$ to be close to median gradient magnitude of the first linear layer which in our experiments was also $\approx 1.0  \!\times\! 10^{-4}$. We chose the range for $C$ such that for the upper bound $2$, most individual input gradients are not clipped, while for the lower bound $1$ almost all are. 

\paragraph{Adapting \tool to Noisy Gradients} In the experiments presented in \cref{tab:dp}, we make several adjustments to \tool to better handle the noise added by DPSGD. First, we apply looser thresholds in our sparsity filtering at \cref{alg_l:sparsity_filter} in \cref{alg:gradient_dismantle} to account for the noise added to the sparse entries of $\frac{\partial \mathcal{L}} {\partial \mZ}$. To account for the imperfect reconstructions in this setting, we also perform our early stopping (\cref{alg_l:early_stop} in \cref{alg:gradient_dismantle}) when the sparsity matching coefficient $\gamma$ reaches a lower value than $1$. Further, we sample matrices $\mL_A$ of larger size (${b+1\times b}$) to increase the numerical stability of our solutions under noise. While sampling larger $\mL_A$ is more computationally expensive, as $b+1$ instead of $b-1$ entries in $\mA$ are required to be correctly sampled as $0$, the resulting directions $\vq_i$ are more numerically stable as they are obtained as a solution of an overdetermined system of linear equations. Note that if $\mA$ is assumed to be of rank $b-1$, \cref{theorem:dirs} remains valid for these larger matrices $\mL_A$. Finally, due to our looser sparsity filtering described above we encounter more incorrect directions $\overline{\vq}_i$. We tackle this issue by only keeping $\overline{\vq}_i$ that correspond to matrices $\mL_A$ of rank exactly $b-1$. Under our assumption in \cref{theorem:dirs}, those are exactly the vectors $\overline{\vq}_i$ that correspond to $\mA$ of the correct rank $b-1$. Note that we apply these changes only for $\sigma>0$.

\paragraph{Invariance to Gradient Clipping} In \cref{tab:dp}, we observe that the quality of our reconstructions is not affected by the clipping constant $C$. This is not a coincidence, but rather a mathematical fact. To see this, note that the observed gradients w.r.t. $\mW$ under clipping are given by:
 \begin{equation*}
 	 \dot{\frac{\partial \mathcal{L}}{\partial \mW}} = \sum_{i=1}^b c_i \frac{\partial \mathcal{L}}{\partial \mW_i} = \sum_{i=1}^b c_i \frac{\partial \mathcal{L}}{\partial \mZ_i} \mX_i,
 \end{equation*}
 where $c_i\in \mathbb{R}$ are the unknown to the attacker factors applied by the clipping procedure to each individual input gradient $\frac{\partial \mathcal{L}}{\partial \mW_i}$. One can adapt the proof to \cref{theorem:dLdW_dLdZ.XT}, to show that   $\dot{\frac{\partial \mathcal{L}}{\partial \mW}} = \dot{\frac{\partial \mathcal{L}}{\partial \mZ}} \mX^T$, where we define $\dot{\frac{\partial \mathcal{L}}{\partial \mZ_i}}$ to be the clipped gradient w.r.t $\mZ$, consisting of the columns $\dot{\frac{\partial \mathcal{L}}{\partial \mZ_i}} = c_i\frac{\partial \mathcal{L}}{\partial \mZ_i}$. We also observe that one can adapt \cref{lemma:bias} to work directly on the clipped gradients as well, resulting in the formula $\dot{\frac{\partial \mathcal{L}} {\partial \vb}} = \dot{\frac{\partial \mathcal{L}} {\partial \mZ}} \begin{bsmallmatrix} 1 \\ \vphantom{\int\limits^x}\smash{\vdots} \\ 1 \end{bsmallmatrix}$ for the clipped gradient w.r.t. $\vb$. The formula follows from the observation that in our setting the same clipping factor $c_i$ is applied to the gradients of each layer, including $\frac{\partial \mathcal{L}}{\partial \vb\indices{_i}}$ and $\frac{\partial \mathcal{L}}{\partial \mW_i}$. By applying the rest of the theoretical results of the paper without change but on clipped gradients $\dot{\frac{\partial \mathcal{L}}{\partial \mZ}}$, instead of the original unclipped gradients $\frac{\partial \mathcal{L}} {\partial \mZ}$, we conclude that \tool is directly applicable on the clipped client gradient and that applying it on those still recover the true input matrix $\mX$ without the need of knowing the clipping constants $c_i$.

\paragraph{Robustness to Noise} From \cref{tab:dp}, we observe that \tool is very robust to noise. We emphasize in particular that even when noise of similar size to the size of the gradients in expectation is applied, we still obtain a reconstruction with PSNR $>28$. This is similar to the PSNR of 29.3 that \citet{geiping} achieves *without any noise* which is commonly considered unacceptable information leakage. These experiments suggest that to efficiently defend against \tool using noise, one needs to apply such high magnitudes that training will likely be significantly impeded.

\begin{table}[t]
	\centering
	\vspace{-4mm}
	\caption{Reconstruction quality across 100 batches of size $b=20$ computed on \TIN for gradients computed with DPSGD~\cite{dpsgd} with different noise levels $\sigma$ and gradient clipping levels $C$.}
	\label{tab:dp}
	\centering
	\vspace{1mm}
	\resizebox*{0.58\linewidth}{!}{
		\begin{tabular}{@{}
				lcccrcc@{}}
			\toprule
			Method & $C$ & $\sigma$ & PSNR $\uparrow$ & Acc (\%) $\uparrow$ \\
			\midrule
			Geiping et. al~\citep{geiping} & $0.00$ & $0$ & $29.3$ & $100$ \\
			\midrule
			\tool(Ours) & $1.00$ & $0$  & $118.2$ & $100$ \\
			\tool(Ours) & $1.25$ & $0$  & $118.1$ & $100$ \\
			\tool(Ours) & $1.50$ & $0$  & $118.5$ & $100$ \\
			\tool(Ours) & $1.75$ & $0$  & $118.7$ & $100$ \\
			\tool(Ours) & $2.00$ & $0$  & $118.0$ & $100$ \\
			\midrule
			\tool(Ours) & $1.00$ & $5.0  \!\times\! 10^{-6}$  & $38.6$ & $99$ \\
			\tool(Ours) & $1.25$ & $5.0  \!\times\! 10^{-6}$  & $40.4$ & $98$ \\
			\tool(Ours) & $1.50$ & $5.0  \!\times\! 10^{-6}$  & $41.9$ & $98$ \\
			\tool(Ours) & $1.75$ & $5.0  \!\times\! 10^{-6}$  & $42.2$ & $97$ \\
			\tool(Ours) & $2.00$ & $5.0  \!\times\! 10^{-6}$  & $42.0$ & $96$ \\
			\midrule
			\tool(Ours) & $1.00$ & $1.0  \!\times\! 10^{-5}$  & $38.2$ & $99$ \\
			\tool(Ours) & $1.25$ & $1.0  \!\times\! 10^{-5}$  & $40.0$ & $98$ \\
			\tool(Ours) & $1.50$ & $1.0  \!\times\! 10^{-5}$  & $38.5$ & $99$ \\
			\tool(Ours) & $1.75$ & $1.0  \!\times\! 10^{-5}$  & $39.2$ & $99$ \\
			\tool(Ours) & $2.00$ & $1.0  \!\times\! 10^{-5}$  & $39.6$ & $99$ \\
			\midrule
			\tool(Ours) & $1.00$ & $5.0  \!\times\! 10^{-5}$  & $32.3$ & $97$ \\
			\tool(Ours) & $1.25$ & $5.0  \!\times\! 10^{-5}$  & $33.5$ & $98$ \\
			\tool(Ours) & $1.50$ & $5.0  \!\times\! 10^{-5}$  & $34.4$ & $99$ \\
			\tool(Ours) & $1.75$ & $5.0  \!\times\! 10^{-5}$  & $34.6$ & $100$ \\
			\tool(Ours) & $2.00$ & $5.0  \!\times\! 10^{-5}$  & $34.1$ & $100$ \\
			\midrule
			\tool(Ours) & $1.00$ & $1.0  \!\times\! 10^{-4}$  & $29.7$ & $98$ \\
			\tool(Ours) & $1.25$ & $1.0  \!\times\! 10^{-4}$  & $29.3$ & $97$ \\
			\tool(Ours) & $1.50$ & $1.0  \!\times\! 10^{-4}$  & $29.9$ & $99$ \\
			\tool(Ours) & $1.75$ & $1.0  \!\times\! 10^{-4}$  & $29.4$ & $98$ \\
			\tool(Ours) & $2.00$ & $1.0  \!\times\! 10^{-4}$  & $28.7$ & $95$ \\

			\bottomrule
		\end{tabular}
	}
	\vspace{-2em}
\end{table}

\section{\tool under FedAvg Updates} \label{app:fedavg}
In this section, we first demonstrate theoretically that \tool can be generalized to attack FedAvg~\citep{fedavg} client updates, and then present empirical results confirming that \tool is indeed very effective under FedAvg protocols with different number of epochs $\mathcal{E}$, local client learning rates $\eta$, and, even works, when mini-batches of size $b_{\text{mini}}$ are used. 

\paragraph{Generalizing \tool to FedAvg Updates}
Assuming that a client uses all of its data points, $\mX$, in each local gradient step of the FedAvg protocol, i.e. $b_{\text{mini}} = b$, the client computes and subsequently shares with the server the following updated linear layer weights:
\begin{equation*}
	\mW^\mathcal{E} = \mW^0 - \eta\sum_{e=1}^{\mathcal{E}} \frac{\partial\mathcal{L}}{\partial \mW^e}= \mW^0 - \eta\sum_{e=1}^{\mathcal{E}} \frac{\partial\mathcal{L}}{\partial \mZ^e} \cdot \mX^T = \mW^0 - \eta\left(\sum_{e=1}^{\mathcal{E}} \frac{\partial\mathcal{L}}{\partial \mZ^e}\right) \cdot \mX^T,
\end{equation*} 
where $\mW^0$ is the global model sent by the server, $\mW^e$ represent the local client weights after $e$ client epochs, and $\frac{\partial\mathcal{L}}{\partial \mW^e}$ and $\frac{\partial\mathcal{L}}{\partial \mZ^e}$ are the weight and output gradients at epoch $e$. 

We empirically observe that sparsity patterns of the different local gradients $\frac{\partial\mathcal{L}}{\partial \mZ^e}$ are usually similar. This is expected as these patterns correspond to the ReLU activation patterns for the layer outputs $\mZ^e$ (see \cref{{sec:sparsity}}) at different local steps which are computed on the same data $\mX$ and with similar weights $\mW^e$. As the sparsity patterns for the individual gradients are similar, their sum $\sum_{e=1}^{\mathcal{E}} \frac{\partial\mathcal{L}}{\partial \mZ^e}$ also shares this sparsity pattern and is, thus, also sparse. As the server knows $\mW^0$ and it can subtract it from the client’s shared weights $\mW^\mathcal{E}$ and apply \cref{theorem:dirs}, as before, on the sparse matrix $\sum_{e=1}^{\mathcal{E}} \frac{\partial\mathcal{L}}{\partial \mZ^e}$ to obtain the corresponding matrix $\mQ$ and client data $\mX$. We note that while our sparsity matching coefficient $\sigma$ will typically not reach 1 for the final reconstruction in this setting, as there is some mismatch between the sparsity patterns of the different output gradients $\frac{\partial\mathcal{L}}{\partial \mZ^e}$, we have found that \tool remains practically effective regardless. 

We note that \tool can be even be generalized to FedAvg protocols that use random mini-batches $\mX^e$ of size $b_{\text{mini}} < b$ sampled from $\mX$ at each local step. This is the case, as each local client gradient $\frac{\partial\mathcal{L}}{\partial \mW^e} = \frac{\partial\mathcal{L}}{\partial \mZ^e} (\mX^e)^T$, can be represented as $\overline{\frac{\partial\mathcal{L}}{\partial \mZ^e}} \mX^T$, where $\overline{\frac{\partial\mathcal{L}}{\partial \mZ^e}}$ is derived from $\frac{\partial\mathcal{L}}{\partial \mZ^e}$ by adding $0$ columns at batch positions corresponding to batch elements not in $\mX^e$. Importantly, as $\overline{\frac{\partial\mathcal{L}}{\partial \mZ^e}}$ only adds $0$ columns to $\frac{\partial\mathcal{L}}{\partial \mZ^e}$, the sparsity of $\overline{\frac{\partial\mathcal{L}}{\partial \mZ^e}}$ can only increase, allowing to conclude that $\sum_{e=1}^{\mathcal{E}} \overline{\frac{\partial\mathcal{L}}{\partial \mZ^e}}$ remains sparse, and, thus, \cref{theorem:dirs} can still be applied to it.

\paragraph{Experiments with FedAvg Updates}
Next, we show empirically the effectiveness of \tool for FedAvg updates. In \cref{tab:fedavg_epoch}, we show the results of attacking clients with $b=20$ datapoints from the \TIN dataset for different number of local client epochs $\mathcal{E}$. We observe that even for $\mathcal{E}=50$ gradient steps we recover data from most batches, with quality similar to the quality achieved when attacking individual gradients. This is expected as \cref{theorem:dirs} still holds, as described in the previous paragraph. The slight dip in the fraction of reconstructed batches for larger number of steps $\mathcal{E}$ can be attributed to some client batches inducing larger discrepancy between the sparsity patterns of $\frac{\partial\mathcal{L}}{\partial \mZ^e}$ compared to others, resulting in their sum being much less sparse. Further, \cref{tab:fedavg_epoch} also shows that \tool can attack client updates that take $b/b_{\text{mini}}=4$ local steps per epoch for $\mathcal{E}=20$ epochs. Interestingly, while a total of $80$ gradient steps are taken in this scenario the results are closer to the $b_{\text{mini}}=20, \mathcal{E}=20$ setting, instead of the $b_{\text{mini}}=20, \mathcal{E}=50$ setting. This can be explained by the increased sparsity of the individual expanded gradients $\overline{\frac{\partial\mathcal{L}}{\partial \mZ^e}}$. 

Finally, we experiment with different local client learning rates $\eta$ and show the results in \cref{tab:fedavg_lr}. We observe that even for large learning rates \tool still recovers its inputs well, showing that while the individual weights $\mW^e$ can change a lot, their induced sparsity on $\frac{\partial\mathcal{L}}{\partial \mZ^e}$ remains consistent.

\begin{table}[t]
	\centering
	\vspace{-4mm}
	\caption{Reconstruction quality across 100 FedAvg client updates computed on \TIN batches of size $b=20$ for different number of epochs $\mathcal{E}$ and different mini batch sizes $b_{\text{mini}}$.}
	\label{tab:fedavg_epoch}
	\centering
	\vspace{2mm}
	\resizebox*{0.45\linewidth}{!}{
		\begin{tabular}{@{}
				lccrcc@{}}
			\toprule
			$\eta$  & $\mathcal{E}$ & $b_{\text{mini}}$ & PSNR $\uparrow$ & Acc (\%) $\uparrow$ \\
			\midrule
			0.01 & 1  & 20 & $97.8$ & $97$ \\
			0.01 & 5  & 20 & $103.9$ & \textbf{100} \\
			0.01 & 10 & 20 & $106.7$ & $99$ \\
			0.01 & 20 & 20 & \textbf{108.90} & $98$ \\
			0.01 & 50 & 20 & $104.9$ & $90$ \\
			0.01 & 20 & 5  & $106.7$ & $97$ \\
			\bottomrule
		\end{tabular}
	}
\end{table}

\begin{table}[t]
	\centering
	\vspace{-4mm}
	\caption{Reconstruction quality across 100 FedAvg client updates computed on \TIN batches of size $b=20$ for different local client learning rates $\eta$.}
	\label{tab:fedavg_lr}
	\centering
	\vspace{1mm}
	\resizebox*{0.45\linewidth}{!}{
		\begin{tabular}{@{}
				lccrcc@{}}
			\toprule
			$\eta$ & $\mathcal{E}$ & $b_{\text{mini}}$ & PSNR $\uparrow$ & Acc (\%) $\uparrow$ \\
			\midrule
			0.1   & 5 & 20 & \textbf{119.3} & $95$ \\
			0.01  & 5 & 20 & $103.9$ & \textbf{100} \\
			0.001 & 5 & 20 & $85.5$ & \textbf{100} \\
			\bottomrule
		\end{tabular}
	}
\end{table}

\section{Additional Visualisations}

In this section we present additional visualisations of the reconstructions obtained by \tool. First, in \cref{fig:full_batch} we show an extended comparison between the images recovered by our method and \citet{geiping} on the \textsc{TinyImageNet} batch first shown in \cref{fig:accept}. In \cref{fig:full_batch} we operate in the same setting as \cref{tab:same_label}, namely batches of only a single class. We observe that while some images are reconstructed well by \citet{geiping}, most of the images are of poor visual quality, with some even being hard to recognize. In contrast, all of our reconstructions are pixel perfect. This in particular also means, that \tool's reconstructions improve in fine-detail recovery even upon the well recovered images of \citet{geiping}. This is expected as our attack is exact (up numerical errors).

\begin{figure}[t!]
	\centering
	\resizebox{1.03\linewidth}{!}{
		\hspace{-5mm}
		\input{figures/full_comparison_fig.tex}
	}
	\vspace{-6mm}
	\caption{The reconstructions of all images from \cref{fig:accept}, reconstructed using our \tool (top) or the prior state-of-the-art \citet{geiping} (mid), compared to the ground truth (bottom).}
	\vspace{-3mm}
	\label{fig:full_batch}
\end{figure}

Further, to show the results in \cref{fig:full_batch} are representative, in \crefrange{fig:10p}{fig:90p} we provide additional visualizations of the reconstructions obtained by \tool corresponding to the $10^{\text{th}}$,$50^{\text{th}}$, and $90^{\text{th}}$ percentiles of the  PSNRs obtained in the \textsc{TinyImageNet} experiment reported in \cref{tab:main_results}. We observe that only 1 sample has visual artefacts for the $10^{\text{th}}$ percentile batch (top left image in \cref{fig:10p}) and that the $50^{\text{th}}$ and $90^{\text{th}}$ percentile batches contain only perfect reconstructions. We theoreticize that the visual artefact in \cref{fig:10p} is a result of a numerical instability issue and that using $\mL_A$ of bigger size as described in \cref{app:dp} one could further alleviate it in exchange of additional computation.

Finally, we demonstrate what happens to \tool reconstructions in the rare case when the algorithm fails to recover all correct directions $\overline{q}_i$ from the batch gradient. In \cref{fig:spear_fail}, we show the only such batch for the \textsc{TinyImageNet} experiment reported in \cref{tab:main_results}. The batch has 2 wrong directions and still achieves an average PSNR of $91.2$ (the worst PSNR obtained in this experiment), which is still much higher compared to prior work. Further, all but 2 images are affected by the failure.

\begin{figure}[H]
	\centering
	\hspace{-0.5cm}
	\resizebox{1.03\linewidth}{!}{
		\hspace{-5mm}
		\input{figures/spear_10th_percentile_fig.tex}
	}
	\caption{Visualisation of the images reconstructed by \tool from the batch whose PSNR is at the $10^{\text{th}}$ percentile based on the set of 100 \textsc{TinyImageNet} reconstructions reported in \cref{tab:main_results}.}
	\label{fig:10p}
\end{figure}

\begin{figure}[H]
	\centering
	\hspace{-0.5cm}
	\resizebox{1.03\linewidth}{!}{
		\hspace{-5mm}
		\input{figures/spear_50th_percentile_fig.tex}
	}
	\caption{Visualisation of the images reconstructed by \tool from the batch whose PSNR is at the $50^{\text{th}}$ percentile based on the set of 100 \textsc{TinyImageNet} reconstructions reported in \cref{tab:main_results}.}
	\label{fig:50p}
\end{figure}

\begin{figure}[H]
	\centering
	\hspace{-0.5cm}
	\resizebox{1.03\linewidth}{!}{
		\hspace{-5mm}
		\input{figures/spear_90th_percentile_fig.tex}
	}
	\caption{Visualisation of the images reconstructed by \tool from the batch whose PSNR is at the $90^{\text{th}}$ percentile based on the set of 100 \textsc{TinyImageNet} reconstructions reported in \cref{tab:main_results}.}
	\label{fig:90p}
\end{figure}

\begin{figure}[H]
	\centering
	\resizebox{1.03\linewidth}{!}{
		\hspace{-5mm}
		\input{figures/spear_fail_fig.tex}
	}
	\vspace{-6mm}
	\caption{Visualisation of the images reconstructed by \tool from the only batch from the 100 \textsc{TinyImageNet} reconstructions reported in \cref{tab:main_results}, where not all recovered directions $\overline{q}_i'$ are correct. \tool recovered 2/20 wrong directions, resulting in the left most images being wrongly recovered.}
	\vspace{-3mm}
	\label{fig:spear_fail}
\end{figure}

	\fi
\end{document}